\documentclass{article}

\usepackage[final]{neurips_2023}

\usepackage{caption} 
\usepackage[utf8]{inputenc} % allow utf-8 input
\usepackage[T1]{fontenc}    % use 8-bit T1 fonts
\usepackage{hyperref}       % hyperlinks
\usepackage{url}            % simple URL typesetting
\usepackage{booktabs}       % professional-quality tables
\usepackage{amsfonts}       % blackboard math symbols
\usepackage{nicefrac}       % compact symbols for 1/2, etc.
\usepackage{microtype}      % microtypography
\usepackage{xcolor}         % colors
\usepackage{mathextra}
    \usepackage{algorithm2e}
\usepackage[normalem]{ulem}

\usepackage[capitalise]{cleveref}
\crefname{assumption}{Assumption}{Assumptions}  % not sure how to automatically capitalise this..

\author{%
  Johannes Kirschner\\%\thanks{Use footnote for providing further information
  Department of Computer Science\\
  University of Alberta\\
  \texttt{jkirschn@ualberta.ca} \\
   \And
   Seyed Alireza Bakhtiari \\
   Department of Computer Science\\
   University of Alberta \\
   \texttt{sbakhtia@ualberta.ca} \\
   \And
   Kushagra Chandak \\
   Department of Computer Science\\
   University of Alberta \\
   \texttt{kchandak@ualberta.ca} \\
   \And
   Volodymyr Tkachuk \\
   Department of Computer Science\\
   University of Alberta \\
   \texttt{vtkachuk@ualberta.ca} \\
   \And
   Csaba Szepesvári \\
   Department of Computer Science\\
   University of Alberta \\
   \texttt{szepesva@ualberta.ca} \\
}

\newcommand{\rdec}{\text{dec}}
\newcommand{\dec}{\text{dec}}

\newcommand{\acfpacdec}{\text{{\normalfont pac-dec}}^{ac,f}}
\newcommand{\acdec}{\text{{\normalfont dec}}^{ac}}
\newcommand{\acfdec}{\text{{\normalfont dec}}^{ac,f}}
\newcommand{\odec}{\text{{\normalfont dec}}^{o}}
\newcommand{\cdec}{\text{{\normalfont dec}}^{c}}
\newcommand{\dc}{\text{{\normalfont dc}}}

\newcommand{\KL}{D_{\text{KL}}}
\newcommand{\co}{\text{{\normalfont co}}}
\newcommand{\Est}{\text{{\normalfont Est}}}
\newcommand{\EST}{\text{{\normalfont EST}}}

\newcommand{\eps}{\epsilon}
\newcommand{\st}{\qquad\text{s.t.}\qquad}
\newcommand{\AETD}{\textsc{Anytime-E2D}}
\newcommand{\ETD}{\textsc{E2D}}
\newcommand{\UCB}{\textsc{UCB}}
\newcommand{\TSa}{\textsc{TS}}
\newcommand{\ETDp}{\textsc{E2D$^+$}}
\renewcommand{\cH}{\cM}

\usepackage[textsize=tiny]{todonotes}
\setlength{\marginparwidth}{2cm}

\definecolor{darkgreen}{rgb}{0.0, 0.5, 0.0}

\LinesNumbered
\RestyleAlgo{ruled}
\SetKwInOut{Input}{Input}

\title{Regret Minimization via Saddle Point Optimization}

\begin{document}

\maketitle

\begin{abstract}
A long line of works characterizes the sample complexity of regret minimization in sequential decision-making by min-max programs. 
In the corresponding saddle-point game, the min-player optimizes the sampling distribution against an adversarial max-player that chooses confusing models leading to large regret. The most recent instantiation of this idea is the decision-estimation coefficient (DEC), which was shown to provide nearly tight lower and upper bounds on the worst-case expected regret in structured bandits and reinforcement learning. By re-parametrizing the offset DEC with the confidence radius and solving the corresponding min-max program, we derive an anytime variant of the Estimation-To-Decisions algorithm (\AETD). Importantly, the algorithm optimizes the exploration-exploitation trade-off online instead of via the analysis. Our formulation leads to a practical algorithm for finite model classes and linear feedback models. We further point out connections to the information ratio, decoupling coefficient and PAC-DEC, and numerically evaluate the performance of E2D on simple examples. \looseness=-1

\end{abstract}

%!TEX root =  ../neurips_2022.tex
\section{Introduction}

Regret minimization is a widely studied objective in bandits and reinforcement learning theory \citep{lattimore2020bandit} that has inspired practical algorithms, for example, in noisy zero-order optimization\citep[e.g.,][]{srinivas10gaussian} and deep reinforcement learning \citep[e.g.,][]{osband2016deep}. Cumulative regret measures the online performance of the algorithm by the total loss suffered due to choosing suboptimal decisions. 
Regret is unavoidable to a certain extent as the learner collects information to reduce uncertainty about the environment. 
In other words, a learner will inevitably face
the exploration-exploitation trade-off where it must balance collecting rewards and collecting information. Finding the right balance is the central challenge of sequential decision-making under uncertainty.\looseness=-1

%In this paper, we work with the following sequential decision-making problem:
More formally, denote by $\Pi$ a decision space and $\cO$ an observation space. Let $\cH$ be a class of models, where $f = (r_f, M_f) \in \cH$ associated with a reward function $r_f : \Pi \rightarrow \bR$ and observation map $M_f : \Pi \rightarrow \sP(\cO)$, 
where $\sP(\cO)$ is the set of all probability distributions over $\cO$.%
\footnote{To simplify the presentation, we ignore tedious measure-theoretic details in this paper. The reader could either fill out the missing details, or just assume that all sets, unless otherwise stated, are discrete.}
%\todoc{So I said we ignore measure theory. Sometimes this is easy. However, one may need additional structure of one talks about policies obtained with argmin (the question of ``measurable selection''). When things are discrete, all is fine. And also in Polish spaces.. Do not ask me about other settings.}
The learner's objective is to collect as much reward as possible in $n$ steps when facing a model $f^* \in \cH$. The learner's prior information is $\cH$ and the associated reward and observation maps, but does not know the true instance $f^* \in \cH$.
The learner constructs a stochastic sequence $\pi_1, \dots, \pi_n$ of decisions taking values in $\Pi$ and adapted to the history of observations $y_t \sim M_{f^*}(\pi_t)$. 
%: $\pi_t$ is chosen stochastically as a function of $(\pi_1,y_1,\dots,\pi_{t-1},y_{t-1})$. 
The policy of the learner is the sequence of probability kernels $\mu_{1:n} = (\mu_t)_{t=1}^n$ that are used to take decisions.
%Note that the policy of a learner is independent of the true unknown model $f^*$.
The expected regret of a policy $\mu_{1:n}$ and model $f^*$ after $n \in \bN$ steps is 
\begin{align*}
	R_n(\mu_{1:n}, f^*) = \max_{\pi \in \Pi} \EE[\sum_{t=1}^n  r_{f^*}(\pi) - r_{f^*}(\pi_t)]
\end{align*}
The literature studies regret minimization for various objectives, including worst-case and instance-dependent frequentist regret \citep{lattimore2020bandit}, Bayesian regret \citep{russo2014learning} and robust variants \citep{garcelon2020adversarial,kirschner2020distributionally}.
For the frequentist analysis, all prior knowledge is encoded in the model class $\cH$. The worst-case regret of policy $\mu_{1:n}$ on $\cH$ is $\sup_{f \in \cH} R_n(\mu_{1:n}, f)$, and therefore the optimal minimax regret $\inf_\mu \sup_{f \in \cH} R_n(\mu_{1:n}, f)$ only depends on $\cH$ and the horizon $n$. The Bayesian, in addition, assumes access to a prior $\nu \in \sP(\cH)$, which leads to the Bayesian regret $\EE_{f \sim \nu}[R_n(\mu_{1:n}, f)]$. Interestingly, the worst-case frequentist regret and Bayesian regret are dual in the following sense \citep{lattimore2019information}:\footnote{The result by
\citet{lattimore2019information} was only shown for finite action, reward and observation spaces, but can likely be extended to the infinite case under suitable continuity assumptions.} 
% \todoc{also, before our paper there were others noting this in special cases (with incorrect proofs). not sure whether we want to refer to them..}
\begin{align}
	\inf_{\mu_{1:n}} \sup_{f \in \cH} R_n(\mu_{1:n}, f) = \sup_{\nu \in \sP(\cH)} \inf_{\mu_{1:n}} \EE_{f \sim \nu}[R_n(\mu_{1:n}, f)] \label{eq:freq-bayes-regret}
\end{align}
Unfortunately, directly solving for the minimax policy (or the worst-case prior) is intractable, except in superficially simple problems. Ths is because the optimization is over the exponentially large space of adaptive policies.
However, the relationship in \cref{eq:freq-bayes-regret} has been directly exploited in prior works, for example, to derive non-constructive upper bounds on the worst-case regret via a Bayesian analysis \citep{bubeck2015bandit}. Moreover, it can be seen as inspiration underlying ``optimization-based'' algorithms for regret minimization: The crucial step is to carefully relax the saddle point problem in a way that preserves the statistical complexity, but can be analyzed and computed more easily. This idea manifests in several closely related algorithms, including information-directed sampling \citep{russo2014learning,kirschner2018information}, ExpByOpt \citep{lattimore2020exploration,lattimore2021mirror}, and most recently, the \mbox{Estimation-To-Decisions (E2D)} framework \citep{foster2021statistical,foster2023tight}. These algorithms have in common that they optimize the information trade-off directly, which in structured settings leads to large improvements compared to standard optimistic exploration approaches and Thompson sampling. 
% Yet, the precise relation among the different approaches are not yet fully understood. 
On the other hand, algorithms that directly optimize the information trade-off can be computationally more demanding and, consequently, are often not the first choice of practitioners. This is partly due to the literature primarily focusing on statistical aspects, leaving computational and practical considerations underexplored.

%In this work, we refine and extend several

%Most recently,

%This relation can be used to prove upper

%Regret can be studied both in the Bayesian and the frequentist setting. Bayesian regret is the regret of the learner averaged over all the environments with which the learner could interact. Frequentist regret can be \emph{instance-dependent}, capturing the performance of the learner on a specific environment. Or, it can be \emph{worst-case}, capturing the learner's performance against the worst possible environment.

%The focus of this work is worst-case regret, where the goal is to design algorithms that perform well for all models in a given hypotheses class. It is understood that the worst-case regret is a saddle point problem \cite{lattimore2020exploration}, where the dual problem corresponds to the Bayesian regret. However, since the primal problem is the set of all adaptive policies, analyizing the worst-case saddle of wo is intractable in general.

%!TEX root =  ../neurips_2022.tex
\paragraph{Contributions}
Building on the results by \cite{foster2021statistical}, we introduce the \emph{average-constrained decision-estimation coefficient} ($\acdec_\eps$), a saddle-point objective that characterizes the frequentist worst-case regret in sequential decision-making with structured observations. Compared to the decision-estimation coefficient of \citep{foster2021statistical}, the $\acdec_\eps$ is parametrized via the confidence radius $\eps$, instead of the Lagrangian offset multiplier. This allows optimization of the information trade-off online by the algorithm, instead of via the derived regret upper bound. Moreover, optimizing the $\acdec_{\eps}$ leads to an anytime version of the E2D algorithm (\AETD) with a straightforward analysis. We also point out relations between the $\acdec_{\eps}$, the information ratio \citep{russo2016information}, the decoupling coefficient \citep{zhang2022feel} and a PAC version of the DEC \citep{foster2023tight}. We further detail how to implement the algorithm for finite model classes and linear feedback models, and demonstrate the advantage of the approach by providing improved bounds for linear bandits with side-observations. Lastly, we report the first empirical results of the E2D algorithm on simple examples. \looseness=-1
\subsection{Related Work}

There is a broad literature on regret minimization in bandits \citep{lattimore2020bandit} and reinforcement learning \citep{jin2018q, azar2017minimax, zhou2021nearly, du2021bilinear, zanette2020learning}. 
Arguably the most popular approaches are based on optimism, leading to the widely analysed upper confidence bound (UCB) algorithms \citep{lattimore2020bandit}, and Thompson sampling (TS) \citep{thompson1933likelihood,russo2016information}.

A long line of work approaches regret minimization as a saddle point problem. \citet{degenne2020games} showed that in the structured bandit setting, an algorithm based on solving a saddle point equation achieves asymptotically optimal regret bounds, while explicitly controlling the finite-order terms. \citet{lattimore2020exploration} propose an algorithm based on exponential weights in the partial monitoring setting \citep{rustichini1999minimizing} that finds a distribution for exploration by solving a saddle-point problem.  The saddle-point problem balances the trade-off between the exponential weights distribution and an information or stability term. The same approach was further refined by \cite{lattimore2021mirror}. In stochastic linear bandits, \cite{kirschner2021asymptotically} demonstrated that information-directed sampling can be understood as a primal-dual method solving the asymptotic lower bound, which leads to an algorithm that is both worst-case and asymptotically optimal. The saddle-point approach has been further explored in the PAC setting \citep[e.g.,][]{degenne2020games,degenne2020gamification}.

Our work is closely related to recent work by \citet{foster2021statistical, foster2023tight}. 
They consider \textit{decision making with structured observations} (DMSO), which generalizes the bandit and RL setting.
They introduce a complexity measure, the \textit{offset decision-estimation coefficient} (offset DEC), defined as a min-max game between a learner and an environment, and provide lower bounds in terms of the offset DEC. 
Further, they provide an algorithm, \textit{Estimation-to-Decisions} (E2D) with corresponding worst-case upper bounds in terms of the offset DEC.
Notably, the lower and upper bound nearly match and recover many known results in bandits and RL.
% The offset DEC is
% \begin{align}
% 	\rdec^o_\gamma(f) = \min_{\mu \in \sP(\Pi)} \max_{g \in \cH} \mu \Delta e_g - \gamma \mu I_{f}e_g
% \end{align}

More recently, \citet{foster2023tight} refined the previous bounds by introducing the \emph{constrained} DEC
% \begin{align}
% 	\rdec^c_\eps(f) = \min_{\mu \in \sP(\Pi)} \max_{g \in \cH} \mu \Delta e_g \st \mu I_{f}e_g \leq \epsilon^2
% \end{align}
and a corresponding algorithm E2D$^+$. %Although achieving better bounds, this algorithm and the analysis is significantly more involved than for the E2D algorithm.

There are various other results related to the DEC and the E2D algorithm. 
\citet{foster2022note} show that the E2D achieves improved bounds in model-free RL when combined with optimistic estimation (as introduced by \cite{zhang2022feel}).
\citet{chen2022unified} introduced two new complexity measures based on the DEC that are necessary and sufficient for reward-free learning and PAC learning.
They also introduced new algorithms based on the E2D algorithm for the above two settings and various other improvements.
\citet{foster2022complexity} have shown that the DEC is necessary and sufficient to obtain low regret for \textit{adversarial} decision-making. 
An asymptotically instance-optimal algorithm for DMSO has been proposed by \citet{dong2022asymptotic}, extending a similar approach for the linear bandit setting \citep{lattimore2017end}. %However, this approach is not worst-case optimal.
% \citet{zhong2022posterior} introduce a different complexity measure, the \textit{Generalized Eluder Coefficient} (GEC), which was used to provide regret bounds for more general problems such as partially observable MDPs.
%However, they were unable to provide any regret lower bounds, justifying the necessity of the GEC.
%Important, is that none of the above mentioned works address the same issues we are interested in, namely, finding a simple, and practical algorithm for the stochastic setting that can adapt to the problem instance.

The decision-estimation coefficient is also related to the information ratio \citep{russo2014learning} and the decoupling coefficient \citep{zhang2022feel}. The information ratio has been studied under both the Bayesian \citep{russo2014learning} and the frequentist regret \citep{kirschner2018information, kirschner2020linearpm, kirschner2021asymptotically, kirschner2023linear} in various settings including bandits, reinforcement learning, and partial monitoring. The decoupling coefficient was studied for the Thompson sampling algorithm in contextual bandits \citep{zhang2022feel}, and RL \citep{dann2021provably,agarwal2022model}. 

%!TEX root =  ../neurips_2022.tex
\section{Setting}
% The decision-making problem can be viewed as an interactive process between a learner and an environment.
% The environment is first selected from a finite model class $\cH$, and subsequently, the learner interacts with the environment by selecting actions from a finite decision space denoted as $\Pi$.
% The learner then observes feedback in the form of responses from a finite observation space denoted as $\mathcal{O}$.
We consider the sequential decision-making problem already introduced in the preface. Recall that $\Pi$ is a compact decision space and $\cO$ is an observation space. The model class $\cH$ is a set of tuples $f = (r_f, M_f)$ containing a reward function $r_f : \Pi \rightarrow \bR$ and an observation distribution $M_f : \Pi \rightarrow \sP(\cO)$. %Further bounded- and finiteness assumptions will be introduced in the context of the relevant statements.
%We make the assumption that both $\cH$ and $\Pi$ are finite, however, our results extend to continuous action and hypothesis spaces under appropriate technical conditions using standard arguments.
%In particular, the the sample complexity bounds we provide scale with $\log(|\cH|)$ in many cases, facilitating covering arguments for the continuous setting.
%Computationally, we will assume that $\cH$ can be computationally enumerated (or rely on subsampling).
% While this is a strong assumption, it is less clear under what conditions one can achieve better results; and even the case where $\cH$ is finite, computational aspects of E2D have not been explored in the literature.
We define the gap function
\begin{align}
	\Delta(\pi, g) = r_g(\pi_g^*) - r_g(\pi) \, , \nonumber
\end{align}
where $\pi^*_g = \argmax_{\pi \in \Pi} r_g(\pi)$ is an optimal decision for model $g$, chosen arbitrarily if not unique.
A randomized policy is a sequence of kernels $\mu_{1:n} = (\mu_t)_{t=1}^n$ from histories $h_{t-1} = (\pi_1,y_1,\dots,\pi_{t-1},y_{t-1}) \in (\Pi \times \cO)^{t-1}$ to sampling distributions $\sP(\Pi)$. The filtration generated by the history $h_t$ is $\cF_t$.
% adapted to the observed filtration $\cF_t = \{ (\pi_1, y_1), \dots, (\pi_{t-1}, y_{t-1})\}$, $h_{t-1} = (\pi_1,y_1,\dots,\pi_{t-1},y_{t-1})$ 
The learner's decisions $\pi_1, \dots, \pi_n$ are sampled from the policy $\pi_t \sim \mu_t$ and observations $y_t \sim M_{f^*}(\pi_t)$ are generated by an unknown true model $f^* \in \cH$. The expected regret under model $f^*$ is formally defined as follows:
\begin{align}
	R_n(\mu_{1:n}, f^*) = \EE[\sum_{t=1}^n \EE_{\pi_t \sim \mu_t(h_t)}[\Delta(\pi_t, f^*)]] \nonumber
\end{align}
For now, we do not make any assumption about the reward being observed. This provides additional flexibility to model a wide range of scenarios, including for example, duelling and ranking feedback~\citep{YJ09,radlinski08learning,combes15learning,LKLS18,kirschner2021dueling}
 (e.g.~used in reinforcement learning with human feedback, RLHF) or dynamic pricing 
\citep{denBoer15}. 
The setting is more widely known as partial monitoring \cite{rustichini1999minimizing}. 
The special case where the reward is part of the observation distribution is called  \emph{decision-making with structured observations} \cite[DMSO,][]{foster2021statistical}. Earlier work studies the closely related \emph{structured bandit} setting \citep{combes2017minimal}. 

A variety of examples across bandit models and reinforcement learning are discussed in \citep{combes2017minimal,foster2021statistical,foster2023tight,kirschner2023linear}. For the purpose of this paper, we focus on simple cases for which we can provide tractable implementations. Besides the finite setting where $\cM$ can be enumerated, these are the following linearly parametrized feedback models.

\begin{example}[Linear Bandits, \cite{abe1999associative}]\label{ex:linear-bandits}
    The model class is identified with a subset of $\bR^d$ and features $\phi_\pi \in \bR^d$ for each $\pi \in \Pi$. The reward function is $r_f(\pi) = \ip{\phi(\pi), f}$ and the observation distribution is $M_f(\pi) = \cN(\ip{\phi_\pi, f}, 1)$.
%	\begin{align}
%		r_f(\pi) &= \ip{\phi_\pi, f}\,.
%%		I_f(\pi, g) &=  \tfrac{\sigma^2}{2}\ip{\phi_\pi, g - f}^2
%	\end{align}
\end{example}

% \begin{example}[Linear Bandits, \cite{abe1999associative}]\label{ex:linear-bandits}
%     The model class $\cM$ is identified by a set of parameters $\Theta = \{\theta_f \in \bR^d: f \in \cM\}$. 
%     For each $f \in \cM$ the reward function is $r_f(\pi) = \ip{\phi_\pi, \theta_f}$ and the observation distribution is $M_f(\pi) = \cN(\ip{\phi_\pi, \theta_f}, 1)$, where $\phi_\pi \in \bR^d$ is a feature for each $\pi \in \Pi$.
% %	\begin{align}
% %		r_f(\pi) &= \ip{\phi_\pi, f}\,.
% %%		I_f(\pi, g) &=  \tfrac{\sigma^2}{2}\ip{\phi_\pi, g - f}^2
% %	\end{align}
% \end{example}

The linear bandit setting can be generalized by separating reward and feedback maps, which leads to the \emph{linear partial monitoring} framework \citep{lin2014combinatorial,kirschner2020linearpm}. Here we restrict our attention to the special case of \emph{linear bandits with side-observations} \citep[c.f.~][]{kirschner2023linear}, which, for example, generalizes the classical semi-bandit setting \cite{mannor2011bandits}%,caron2012leveraging}.\looseness=-1

%\begin{example}[Linear Partial Monitoring]
%	We have $\cH\subset \bR^d$, and features $\phi_\pi \in \bR^d$ and observation matrices $M_\pi \in \bR^{m \times d}$ for each $\pi \in \Pi$ ans we set $M_f(\pi) = \cN(M_\pi f, \sigma^2 \eye_m)$ and
%	\begin{align}
%		r_f(\pi) &= \ip{\phi_\pi, f}
%%		I_f(g, \pi) &=  \tfrac{\sigma^2}{2}\| M_\pi (g-f)\|_2^2
%	\end{align}
%\end{example}

\begin{example}[Linear Bandits with Side-Observations] \label{ex:semi-bandits}
    As in the linear bandit setting, we have $\cH\subset \bR^d$, and features $\phi_\pi \in \bR^d$ that define the reward functions $r_f(\pi) = \ip{\phi_\pi, f}$. Observation matrices $M_\pi \in \bR^{m_\pi \times d}$ for each $\pi \in \Pi$ define $m_\pi$-dimensional observation distributions $M_f(\pi) = \cN(M_\pi f, \sigma^2 \eye_{m_\pi})$. In addition, we assume that $\phi_\pi \phi_\pi^\top \preceq M_{\pi}^\top M_\pi$, which is automatically satisfied if $\phi_\pi^\top$ is included in the rows of $M_{\pi}$, i.e. when the reward is part of the observations.
\end{example}

%

%!TEX root =  ../neurips_2022.tex
\section{Regret Minimization via Saddle-Point Optimization}

The goal of the learner is to choose decisions $\pi \in \Pi$ that achieve a small gap $\Delta(\pi, f^*)$ under the true model $f^* \in \cH$. Since the true model is unknown, the learner has to collect data that provides statistical evidence to reject models $g \neq f^*$ for which the regret $\Delta(\pi, g)$ is large. To quantify the information-regret trade-off, we use a divergence $D(\cdot\|\cdot)$ defined for distributions in $\sP(\cO)$. For a reference model $f$, the information (or divergence) function is defined by:
\begin{align*}
	I_{f}(\pi, g) = \KL(M_g(\pi)\| M_{f}(\pi))\,,
\end{align*}
where $\KL(\cdot \| \cdot)$ is the KL divergence. Intuitively, $I_f(\pi, g)$ is the rate at which the learner collects statistical information to reject $g \in \cH$ when choosing $\pi \in \Pi$ and data is generated under the reference model $f$. Note that $I_f(\pi, f) = 0$ for all $f\in\cH$ and $\pi \in \Pi$. As we will see shortly, the regret-information trade-off can be written precisely as a combination of the gap function, $\Delta$, and the information function, $I_f$. We remark in passing that other choices such as the Hellinger distance are also possible, and the KL divergence is mostly for concreteness and practical reasons.

%\paragraph{Notation} 
To simplify the notation and emphasize the bilinear nature of the saddle point problem that we study, we will view $\Delta,I_f \in \bR_{+}^{\Pi \times \cH}$ as $|\Pi|\times |\cH|$ matrices (by fixing a canonical ordering on $\Pi$ and $\cH$). 
For vectors $\mu \in \bR^\Pi$ and $\nu \in \bR^\cH$, we will frequently write bilinear forms $\mu \Delta_f \nu$ and $\mu I_f\nu$. This also means that by convention, $\mu$  will always denote a row vector, while $\nu$ will always denote a column vector. 
%Moreover, we will use the same symbols while restricting the domains, e.g. $\nu \in \bR^\cH_{\geq 0}$ or $\nu \in \sP(\cH)$, which we always specify clearly for the relevant context. 
The standard basis for $\bR^\Pi$ and $\bR^\cH$ is $(e_\pi)_{\pi \in \Pi}$ and $(e_g)_{g \in \cH}$.

\subsection{The Decision-Estimation Coefficient}
To motivate our approach, we recall the \emph{decision-estimation coefficient} (DEC) introduced by \cite{foster2021statistical,foster2023tight}, before introducing the main quantity of interest, the \emph{average-constrained DEC}. First, the \emph{offset decision-estimation coefficient} (without localization) \citep{foster2021statistical} is
\begin{align}
	\odec_\lambda(f) = \min_{\mu \in \sP(\Pi)} \max_{g \in \cH} \mu \Delta e_g - \lambda \mu I_{f}e_g \nonumber
\end{align}
The tuning parameter $\lambda>0$ controls the weight of the information matrix relative to the gaps: 
Viewing the above as a two-player zero-sum game, we see that increasing $\lambda$ forces the max-player to avoid models that differ significantly from $f$ under the min-player's sampling distribution. 
The advantage of this formulation is that the information term $\mu I_f e_g$ can be telescoped in the analysis, which directly leads to regret bounds in terms of the estimation error (introduced below in \cref{eq:est}). %$\text{Est}_n = \EE[\sum_{t=1}^n \mu_t I_{\hat f_t}e_{f^*}]$, \todov{this is defined again on next page} where $\hat f_t$ is the model estimated by the learner in round $t$. 
The disadvantage of the $\lambda$-parametrization is that the trade-off parameter is chosen by optimizing the final regret upper bound. This is inconvenient because the optimal choice requires knowledge of the horizon and a bound on $\max_{f \in \cH} \rdec^o_{\lambda}(f)$. Moreover, any choice informed by the upper bound may be conservative, leading to sub-optimal performance.

The \emph{constrained decision-estimation coefficient} \citep{foster2023tight} is
\begin{align}
	\cdec_\eps(f) = \min_{\mu \in \sP(\Pi)} \max_{g \in \cH} \mu \Delta e_g \st \mu I_{f}e_g \leq \epsilon^2 \label{eq:dec-c}
\end{align}
In this formulation, the max player is restricted to choose models $g$ that differ from $f$ at most by $\eps^2$ in terms of the observed divergence under the min-player's sampling distribution. 
Note that because $e_\pi I_f e_f = 0$ for all $e_\pi \in \Pi$, there always exists a feasible solution. 
For horizon $n$, the radius can be set to $\eps^2 \approx \frac{\beta_{\cH}}{n}$, where $\beta_{\cH}$ is a model estimation complexity parameter, thereby essentially eliminating the trade-off parameter from the algorithm. However, because of the hard constraint, strong duality of the Lagrangian saddle point problem (for fixed $\mu$) fails, and consequently, telescoping the information gain in the analysis is no longer easily possible (or at least, with the existing analysis). To achieve sample complexity  $\rdec^c_{\eps}(f)$, \citet{foster2023tight} propose a sophisticated scheme that combines phased exploration with a refinement procedure (\ETDp).
% While the \ETDp~algorithm is an important milestone towards characterizing the achievable sample complexity in structured regret minimization settings, it is an open problem to design a simpler, more practical algorithm that provably achieves the correct scaling in terms of $\rdec^c_\eps(f)$.

As the main quantity of interest in the current work, we now introduce the \emph{average-constrained decision-estimation coefficient}, defined as follows:
\begin{align}
	\acdec_\eps(f) = \min_{\mu \in \sP(\Pi)} \max_{\nu \in \sP(\cH)} \mu \Delta \nu \st \mu I_{f} \nu \leq \epsilon^2  \label{eq:d*}
\end{align}
Similar to the $\cdec_\eps$, the parameterization of the $\acdec_\eps$ is via the confidence radius $\eps^2$, making the choice of the hyperparameter straightforward in many cases (more details in \cref{ss:e2d}).  
By convexifying the domain  $\sP(\cH)$ of the max-player, we recover strong duality of the Lagrangian (for fixed $\mu$). 
Thereby, the formulation inherits the ease of choosing the $\eps$-parameter from the $\cdec_\eps$, while, at the same time, admitting a telescoping argument in the analysis and a much simpler algorithm.

Specifically, Sion's theorem implies three equivalent Lagrangian representations for \cref{eq:d*}:
\begin{align}
	\acdec_\eps(f) &= \min_{\mu \in \sP(\Pi)} \max_{\nu \in \sP(\cH)} \min_{\lambda \geq 0} \mu \Delta \nu - \lambda (\mu I_{f} \nu - \epsilon^2) \label{eq:L1}\\
	&= \min_{\lambda \geq 0, \mu \in \sP(\Pi)} \max_{\nu \in \sP(\cH)} \mu \Delta \nu - \lambda (\mu I_{f} \nu - \epsilon^2)\label{eq:L2}\\
	&= \min_{\lambda \geq 0} \max_{\nu \in \sP(\cH)} \min_{\mu \in \sP(\Pi)}  \mu \Delta \nu - \lambda (\mu I_{f} \nu - \epsilon^2)  \label{eq:L3}
\end{align}
When fixing the outer problem, strong duality holds for the inner saddle-point problem in each line, however, the joint program in \cref{eq:L2} is not convex-concave. An immediate consequence of relaxing the domain of the max player and \cref{eq:L2} is that 
\begin{align}
 \rdec^c_{\eps}(f) \leq \acdec_{\eps}(f) =  \min_{\lambda \geq 0}  \{\rdec^o_\lambda(f) + \lambda\epsilon^2\}  \label{eq:different-decs}
\end{align}
The $\acdec_\eps$ can therefore be understood as setting the $\lambda$ parameter of the $\odec_\lambda$ optimally for the given confidence radius $\eps^2$. On the other hand, the cost paid for relaxing the program is that there exist model classes $\cH$ where the inequality in \cref{eq:different-decs} is strict, and $\acdec_{\eps}$ does not lead to a tight characterization of the regret \cite[Proposition 4.4]{foster2023tight}. The remedy is that under a stronger regularity condition and localization, the two notions are essentially equivalent \cite[Proposition 4.8]{foster2023tight}. \looseness=-1

\subsection{Anytime Estimation-To-Decisions (Anytime-E2D)}\label{ss:e2d}

Estimations-To-Decisions (E2D) is an algorithmic framework that directly leverages the decision-estimation coefficient for choosing a decision in each round. The key idea is to compute a sampling distribution $\mu_t \in \sP(\Pi)$ attaining the minimal DEC for an estimate $\hat f_t$ of the underlying model, and then define the policy to sample $\pi_t \sim \mu_t$. 
The E2D approach, using the $\acdec_\eps$ formulation, is summarized in \cref{alg:e2d}. To compute the estimate $\hat f_t$, the E2D algorithm takes an abstract estimation oracle $\EST$ as input, that, given the collected data, returns $\hat f_t \in \cM$. The final guarantee depends on the \emph{estimation error} (or estimation regret), defined as the sum over divergences of the observation distributions under the estimate $\hat f_t$ and the true model $f^*$:
\begin{align}
	\Est_n = \EE[\sum_{t=1}^n \mu_t I_{\hat f_t} e_{f^*}] \label{eq:est}
\end{align}
Intuitively, the estimation error is well-behaved if $\hat f_t \approx f^*$, since $\mu_t I_{f^*} e_{f^*} = 0$. \Cref{eq:est} is closely related to the \emph{total information gain} used in the literature on information-directed sampling \citep{russo2014learning} and kernel bandits \citep{srinivas10gaussian}.

To bound the estimation error, \cite{foster2021statistical} rely on \emph{online density estimation} (also, \emph{online regression} or \emph{online aggregation}) \citep[Chapter 9]{cesa2006prediction}. For finite $\cM$, the default approach is the \emph{exponential weights algorithm} (EWA), which we provide for reference in \cref{app:ewa}. When using this algorithm, the estimation error always satisfies $\Est_n \leq \log(|\cH|)$, see \cite[Proposition 3.1]{cesa2006prediction}. While these bounds extend to continuous model classes via standard covering arguments, the resulting algorithm is often not tractable without additional assumptions. For linear feedback models (\cref{ex:linear-bandits,ex:semi-bandits}), one can rely on the more familiar ridge regression estimator, which, we show, achieves bounded estimation regret $\Est_n \leq \cO(d \log(n))$. For further discussion, see \cref{app:ridge}.

\begin{algorithm}[t]
	\DontPrintSemicolon
	\SetAlgoVlined
	\SetAlgoNoLine
	\SetAlgoNoEnd

	\caption{\AETD} \label{alg:e2d}
	\Input{ Hypothesis class $\cH$, estimation oracle $\EST$, sequence $\eps_t \geq 0$, data $\cD_0 = \emptyset$}
	\For{$t=1,2,3, \dots$}{
		Estimate $\hat f_t = \EST(\cD_{t-1})$\;
		Compute gap and information matrices, $\Delta$ and $I_{\hat f_t} \in \bR^{\Pi \times \cH}$\;
		% With $\lambda_t$ given: $\mu_t = \argmin_{\mu \in \sP(\Pi)} \max_{\nu \in \sP(\cH)} \mu \Delta \nu - \lambda_t \mu I_{\hat f_t}\nu$\; 
		$\mu_t = \argmin_{\mu \in \sP(\Pi)} \max_{\nu \in \sP(\cH)} \{ \mu \Delta \nu : \mu I_{\hat f_t} \nu \leq \eps_t^2$\}\;
		Sample $\pi_t \sim \mu_t$ and observe $y_t \sim M_{f^*}(\pi_t)$\;
		Append data $\cD_t = \cD_{t-1} \cup \{(\pi_t, y_t)\}$\;
	}
\end{algorithm}

With this in mind, we state our main result.
\begin{theorem}\label{thm:worst-case}
	Let $\lambda_t \geq 0$ be any sequence adapted to the filtration $\cF_t$.
	Then the regret  of \AETD~(\cref{alg:e2d})  with input sequence $\lambda_t$ satisfies for all $n \geq 1$:
	\begin{align*}
		%		R_n &\leq \sum_{t=1}^n d_{\eps_t}(\hat f_t) + \max_{t \in [n]} \max_{f \in \cH} \eps_t^{-2} d_{\eps_t}^*(f) \Est_n + \sqrt{n \Est_n}\\
		R_n &\leq  \esssup_{t \in [n]} \left\{\frac{ \acdec_{\eps_t, \lambda_t}(\hat f_t)}{\eps_t^2}\right\} \left(\sum_{t=1}^n \eps_t^2 + \Est_n\right)  %+ \sqrt{n \Est_n}
	\end{align*}
	where we defined $\acdec_{\eps,\lambda}(f) = \min_{\mu \in \sP(\Pi)} \max_{\nu \in \sP(\cM)} \mu \Delta \nu - \lambda (\mu I_f \nu - \eps^2)$.
\end{theorem}
As an immediate corollary, we obtain a regret bound for \cref{alg:e2d} where the sampling distribution $\mu_t$ is chosen to optimize $\acdec_{\eps_t}$ for any sequence $\eps_t$. 
\begin{corollary}\label{cor:worst-case}
	The regret of \AETD~(\cref{alg:e2d})  with input $\eps_t \geq 0$ satisfies for all $n \geq 1$:
	\begin{align*}
		%		R_n &\leq \sum_{t=1}^n d_{\eps_t}(\hat f_t) + \max_{t \in [n]} \max_{f \in \cH} \eps_t^{-2} d_{\eps_t}^*(f) \Est_n + \sqrt{n \Est_n}\\
	R_n &\leq  \max_{t \in [n], f \in \cH} \left\{\frac{ \acdec_{\eps_t}(f)}{\eps_t^2}\right\} \left(\sum_{t=1}^n \eps_t^2 + \Est_n\right) %+ \sqrt{n \Est_n}
	\end{align*}
\end{corollary}
Importantly, the regret of \cref{alg:e2d} is directly controlled by the worst-case DEC, $\max_{f \in \cH} \acdec_\eps(f)$, and the estimation error $\Est_n$. It remains to set $\eps_t^2$ (respectively $\lambda_t$) appropriately. For a fixed horizon $n$, we let $\eps_t^2 = \frac{\Est_n}{n}$. With the reasonable assumption that $\max_{f \in \cH} \left\{\eps^{-2}  \acdec_\eps(f) \right\}$ is non-decreasing in $\eps$, \cref{cor:worst-case} reads \looseness=-1
\begin{align}
	R_n \leq 2 n \max_{f \in \cH} \left\{\acdec_{\sqrt{\Est_n/n}}(f)\right\} \,.\label{eq:thm-fixed-horizon}
\end{align}
This almost matches the lower bound $R_n \geq \Omega(n \cdec_{1/\sqrt{n}}(\cF))$\footnote{Here, $\cdec_{\eps}(\cF) = \max_{f \in \co(\cH)} \min_{\mu \in \sP(\Pi)} \max_{g \in \cH \cup \{f\}} \{\mu \Delta \nu : \mu I_f e_g \leq \eps^2\}$.} \cite[Theorem 2.2]{foster2023tight}, up to the estimation error and the beforehand mentioned gap between $\cdec_\eps$ and $\acdec_{\eps}$. 

%\begin{remark}[Anytime]
%\label{rem:setting-eps}
	To get an anytime algorithm with essentially the same scaling as in \cref{eq:thm-fixed-horizon}, we set $\eps_t^2 = \log(|\cM|)/t$ for finite model classes, and $\eps_t^2 = \frac{\beta_\cH}{t}$ if $\Est_t \leq \beta_{\cH} \log(t)$ for $\beta_{\cH} > 0$.
%\end{remark}
For linear bandits, $\acdec_\eps \leq \eps \sqrt{d}$ (see \cref{ss:certifying}), and $\Est_n \leq d \log(n)$. Choosing $\eps_t^2 = d/t$ recovers the optimal regret bound $R_n \leq \tilde \cO(d \sqrt{n})$ \citep{lattimore2020bandit}.
 Alternatively, one can also choose $\lambda_t$ by minimizing an upper bound on $\max_{t \in [n], f \in \cH} \left\{{ \acdec_{\eps_t, \lambda_t}(f)}/{\eps_t^2}\right\}$. For example, in linear bandits, $\acdec_{\eps_t, \lambda} \leq \frac{d}{4\lambda} + \lambda \eps_t^2$ (see \cref{tab:regret-bounds}); hence, for $\eps_t^2 = d/t$, we can set $\lambda_t = t/4$. % This avoids the computationally cumbersome optimization over $\lambda$, while giving the same worst-case upper-bound. 
% In other words, this is the anytime version of E2D that onl4y requires to solve $\dec^o_{\lambda_t}(\hat f_t)$. 
% However, it is likely that one can construct cases where this particular upper bound is quite loose.  
Further discussion and refined upper bound for linear feedback models are in \cref{ss:certifying}.

\begin{proof}[Proof of \cref{thm:worst-case}]
	Let $\mu_t^*$ and $\nu_t^*$ be a saddle-point solution to the offset dec,
	$$\dec^o_{\lambda_t}(\hat f_t) = \min_{\mu \in \sP(\Pi)} \max_{\nu \in \sP(\cM)} \mu \Delta \nu - \lambda_t \mu I_{\hat f_t} \nu$$ 
	Note that $\mu_t^* \Delta \nu_t^* - \lambda_t \mu_t^* I_f \nu_t^* \geq \mu_t^* \Delta e_f - \lambda_t \mu_t^* I_f e_f \geq 0$, which implies that $\lambda_t \eps_t^2 \leq \acdec_{\eps_t, \lambda_t}$. % In other words, the conclusion of \cref{lem:d_eps_bounds_2} holds for any $\lambda_t$.
	Next,% we bound the regret:
	\begin{align}
		R_n =  \EE[\sum_{t=1}^n \mu_t \Delta e_{f^*}] &=  \sum_{t=1}^n \EE[\mu_t \Delta e_{f^*} - \lambda_t (\mu_t I_{\hat f_t} e_{f^*} - \eps_t^2) + \lambda_t (\mu_t I_{\hat f_t} e_{f^*} - \eps_t^2)]  \nonumber\\
		&\leq \sum_{t=1}^n \EE[\max_{g \in \cH} \mu_t \Delta_{\hat f_t} e_{g} - \lambda_t (\mu_t I_{\hat f_t} e_{g} - \eps_t^2) + \lambda_t (\mu_t I_{\hat f_t} e_{f^*} - \eps_t^2)]  \nonumber \\
		&= \sum_{t=1}^n  \EE[ \min_{\mu \in \sP(\Pi)} \max_{\nu \in \sP(\cH)}\mu \Delta \nu  - \lambda_t (\mu I_{\hat f_t} \nu - \eps_t^2)  + \lambda_t (\mu_t I_{\hat f_t} e_{f^*} - \eps_t^2)] \nonumber 
	\end{align}
	So far, we only introduced the saddle point problem by maximizing over $f^*$. The last equality is by our choice of $\lambda_t$ and $\mu_t$, and noting that $\nu \in \sP(\cH)$ can always be realized as a Dirac. Continuing, 
	\begin{align*}
		R_n &\leq \sum_{t=1}^n  \EE[\acdec_{\eps_t, \lambda_t}(\hat f_t) + \lambda_t (\mu_t I_{\hat f_t} e_{f^*} - \eps_t^2)] \\
%		&\stackrel{(i)}{\leq} \sum_{t=1}^n  \EE[\acdec_{\eps_t, \lambda_t}(\hat f_t) + \lambda_t \mu_t I_{\hat f_t} e_{f^*}] \\
		&\stackrel{(i)}{\leq} \sum_{t=1}^n \EE[\acdec_{\eps_t, \lambda_t}(\hat f_t) +   \frac{1}{\eps_t^2} \acdec_{\eps_t, \lambda_t}(\hat f_t) \mu_t I_{\hat f_t} e_{f^*}]  \\
		%			&\stackrel{(ii)}{\leq} \sum_{t=1}^n \EE[\frac{1}{\eps_t^2}\acdec_{\eps_t}(\hat f_t)\big(\eps_t^2 +  \mu_t I_{\hat f_t} e_{f^*}\big)]  \label{eq:worst-case-regret-proof-2} \\
		&\stackrel{(ii)}{\leq} \esssup_{t \in [n]}  \max_{f \in \cH} \left\{\frac{1}{\eps_t^2} \acdec_{\eps_t, \lambda_t}(f) \right\} \sum_{t=1}^n \big(\eps_t^2 +  \EE[\mu_t I_{\hat f_t} e_{f^*}]\big)  
		%			&= \max_{t \in [n]}  \max_{f \in \cH}\left\{\frac{1}{\eps_t^2} \acdec_{\eps_t}(f) \right\} \left(\sum_{t=1}^n \eps_t^2  + \Est_n\right) \nonumber \,.
	\end{align*}
	We first drop the negative term in $(i)$ and use the beforehand stated fact that $\lambda_t \eps_t^2 \leq \acdec_{\eps_t, \lambda_t}(\hat f_t)$. The last step, $(ii)$, is taking the maximum out of the sum.
\end{proof}

\begin{table}[t]
	\centering
	\def\arraystretch{1.3}%
	\begin{tabular}{|c|c|c|}
		\hline
		Setting & $\odec_\gamma$ & $\acdec_\eps$ \\ 
		\hline \hline
		Multi-Armed Bandits & $|\Pi|/\gamma$ & $2 \eps \sqrt{|\Pi|}$\\
		\hline
		Linear Bandits & $d/4\gamma$ & $\eps \sqrt{d}$\\
		%        \hline
		%        Linear Semi-Bandits & $??$ & $\min \{\eps \sqrt{d}, \eps^{2/3} d^{1/3}\}$\\
		\hline
		Lipschitz Bandits & $2\gamma^{-\frac{1}{d+1}}$ & $2^{\frac{d+1}{d+2}}\eps^{\frac{2}{d+2}}$ \\
		\hline
		Convex Bandits & $\tilde O (d^4/\gamma)$ & $\tilde O (\eps d^2)$\\
		\hline
	\end{tabular}
	\caption{Comparison of $\odec_\gamma$ and $\acdec_\eps$ for different settings. Bounds between $\odec_\gamma$ and $\acdec_\eps$ can be converted  using \cref{eq:different-decs}. 
%		$\tilde O (\cdot)$ denotes hidden polylog factors.
	Refined bounds for linear bandits with side-observations are in \cref{lem:dec-bounds-linear}.
	}
	\label{tab:regret-bounds}
	\vspace{-10pt}
\end{table}%\todoj{in Table 1, I added the min for Linear bandits. Do we need the same for $\odec$? Vlad:I think would be good to add min for DEC but I'm not sure how to calculate it for DEC right away. We could just add a footnote saying similar min result can likely be calculated for DEC?}

\subsection{Certifying Upper Bounds}\label{ss:certifying}

As shown by \cref{cor:worst-case}, the regret of \cref{alg:e2d} scales directly with the $\acdec_\eps$. 
For analysis purposes, it is however useful to compute upper bounds on the $\acdec_\eps$ to verify the scaling w.r.t.~parameters of interest. 
Via the equivalence \cref{eq:different-decs}, bounds on the $\odec_\lambda$ directly translate to the $\acdec_\eps$ (see \cref{tab:regret-bounds}). 
For a detailed discussion of upper bounds in various models, we refer to \cite{foster2021statistical}.
Below, we highlight three connections that are directly facilitated by the $\acdec_\eps$.
% In particular, we show a hierarchy of \emph{decoupling} arguments (\cite{russo2014learning}) that lead to increasingly weaker bounds on the $\acdec$. 

To this end, we first introduce a variant of the $\acdec_\eps$ where the gap function depends on $f$:
\begin{align}
	\acfdec_\eps(f) = \min_{\mu \in \sP(\Pi)} \max_{\nu \in \sP(\cH)} \mu \Delta_f \nu \st \mu I_{f} \nu \leq \epsilon^2 \,, \label{eq:d*-Delta_f}
\end{align}
where $\Delta_f(\pi, g) = r_g(\pi_g^*) - r_f(\pi)$.
We remark that for distributions $\nu \in \sP(\cH)$ and $\mu \in \sP(\Pi)$, the gap $\Delta_f$ can be decoupled, $\mu \Delta_f \nu = \delta_f \nu + \mu \Delta_f e_f$, where we defined $\delta_f(g) = r_g(\pi_g^*) - r_f(\pi_f^*)$. 
The following assumption implies that the observations for a decision $\pi$ are at least as informative as observing the rewards.

% We remark that the gap $\Delta_f$ can be decoupled as $\mu \Delta_f \nu = \delta_f \nu + \mu \Delta_f e_f$, where we defined $\delta_f(g) = r_g(\pi_g^*) - r_f(\pi_f^*)$. The following assumption implies that the observations for an action $\pi$ are at least as informative as observing the rewards.
%
%This choice additively decouples the reference model $f$ and the alternative model $g$, making $\Delta_f$ the sum of two rank-one matrices. More explicitly, we denote $\delta_f(g) = r_g(\pi_g^*) - r_f(\pi_f^*)$, so that we get
%\begin{align}
%	\Delta_f(\pi, g) = \delta_f(g) + \Delta_f(\pi, f)
%\end{align}
%For distributions $\nu \in \sP(\cH)$ and $\mu \in \sP(\Pi)$, we further get $\mu \Delta_f \nu = \delta_f \nu + \mu \Delta_f e_f$.

\begin{assumption}[Reward Data Processing]\label{asm:reward-data-processing}
	The rewards and information matrices are related via the following data-processing inequality that holds for any $\mu \in \sP(\Pi)$:
	\begin{align}
		|\EE_{\pi \sim \mu}[r_f(\pi) - r_g(\pi)]| \leq \sqrt{\EE_{\pi \sim \mu}[D(M_f(\pi)\|M_g(\pi))]} \nonumber
	\end{align}
\end{assumption}
% \todoj{explain how this corresponds to a ``reward is observed'' assumption}

The next lemma shows that under \cref{asm:reward-data-processing}, $\acdec_\eps(f)$ and $\acfdec_\eps(f)$ are essentially equivalent, at least for the typical worst-case bounds where $\max_{f \in \cH} \acdec_{\eps}(f) \geq \Omega(\eps)$. 
%This is because the best scaling for models with sub-Gaussian observation distribution is $\acdec_{\eps} \leq \cO(\eps \sqrt{C_\cH})$, which leads $\cO(\sqrt{n C_\cH}))$ regret, where $C_{\cH}>0$ is a constant that depends on the function class.
\begin{lemma}\label{lem:dec-Delta_f-comparison}
	If \cref{asm:reward-data-processing} holds, then
	\begin{align*}
		\acfdec_{\eps}(f) - \eps \leq  \acdec_{\eps}(f) \leq \acfdec_{\eps}(f) + \eps 
	\end{align*}
\end{lemma}
The proof is in \cref{app:proof-lem-dec-Delta_f-comparison}.
We remark that \cref{alg:e2d} where the sampling distribution is computed for $\acfdec_{\eps}(\hat f_t)$ and $\Delta_f$ achieves a bound analogous to \cref{thm:worst-case}, as long as \cref{asm:reward-data-processing} holds. 
For details see \cref{lem:worst-case-deltaf} in \cref{app:coef-relations}.

\paragraph{Upper Bounds via Decoupling}
First, we introduce the \emph{information ratio}, % $\Psi_{f}(\mu, \nu) = \frac{(\mu \Delta_f \nu)^2}{\mu I_f \nu}$.
\begin{align}
	\Psi_{f}(\mu, \nu) = \frac{(\mu \Delta_f \nu)^2}{\mu I_f \nu} \nonumber
\end{align}
The definition is closely related to the Bayesian information ratio \citep{russo2016information}, where $\nu$ takes the role of a prior over $\cH$. 
% To relate the $\acfdec$, information ratio, and the decoupling coefficient~\citep[Definition 1]{zhang2022feel}, we first define t
The Thompson sampling distribution is $\mu_\nu^\text{TS} = \sum_{h \in \cH} \nu_h e_{\pi^*_h}$. The decoupling coefficient, $\dc(f)$, \citep[Definition 1]{zhang2022feel} is defined as the smallest number $K \geq 0$, such that for all distributions $\nu \in \sP(\cH)$,
\begin{align}
\label{eq:dc-inq}
    \mu_\nu^\text{TS} \Delta_f \nu 
    \leq \inf_{\eta \geq 0}\bigg\{ \eta \sum_{g, h \in \cH}  \nu_g \nu_h e_{\pi^*_h} (r_g - r_f)^2 + \frac{K}{4\eta}\bigg\}
    = \sqrt{K \textstyle\sum_{g, h \in \cH} \nu_g \nu_h e_{\pi^*_h} (r_g - r_f)^2}
\end{align}
% In other words, the decoupling coefficient is equal to the information ratio for the Thompson sampling distribution and the worst-case prior, $\dc(f) = \max_{\nu \in \sP(\nu)} \Psi_f(\mu_\nu^\TS)$.

The next lemma provides upper bounds on the $\acdec_{\eps}(f)$ in terms of the information ratio, which is further upper-bounded by the decoupling coefficient.
\begin{lemma}
\label{lem:d*-to-dc}
	With $\Psi(f) = \max_{\nu \in \cH} \min_{\mu \in \sP(\Pi)}\Psi_f(\mu, \nu)$ and \cref{asm:reward-data-processing} satisfied, we have
	\begin{align*}
		\acfdec_{\eps}(f) \leq \eps \sqrt{\Psi(f)} \leq \eps \sqrt{\dc(f)}
	\end{align*}
	 % and $\dc(f)$ is the decoupling coefficent \cite[Def 1]{zhang2022feel}
\end{lemma}
The proof follows directly using the AM-GM inequality, see \cref{app:d*-to-dc-proof}. By \cite[Lemma 2]{zhang2022feel}, this further implies $\acfdec_\eps \leq \eps \sqrt{d}$. An analogous result for the generalized information ratio \citep{lattimore2021mirror} that recovers rates $\eps^{\rho}$ for $\rho \leq 1$ is given in \cref{app:generalized-information-ratio}. 

\paragraph{PAC to Regret} Another useful way to upper bound the $\acfdec_\eps$ is via an analogous definition for the PAC setting \cite[c.f. Eq. (10),][]{foster2023tight}:
\begin{align}
    \label{eq:pac-dec}
    \acfpacdec_\eps(f) = \min_{\mu \in \sP(\Pi)} \max_{\nu \in \cH} \delta_f \nu \st \mu I_f \nu \leq \eps^2
\end{align}
\begin{lemma}
	\label{lem:p-dec bound}
      % Define $\acpacdec_\eps(f) = \min_{\mu \in \sP(\Pi)} \max_{\nu \in \cH} \{ \delta_f \nu : \mu I_f \nu \leq \eps^2\}$. 
      Under \cref{asm:reward-data-processing},
	\begin{align*}
		\acfdec_\eps(f) \leq \min_{p \in [0,1]} \left\{ \acfpacdec_{{\eps}p^{-1/2}}(f) + p \Delta_{\max}\right\} 
	\end{align*}
\end{lemma}
The proof is given in \cref{app:proof-pac-to-regret}. \Cref{lem:p-dec bound} combined with \cref{thm:worst-case} leads to $\cO(n^{2/3})$ upper bounds on the regret that are reminiscent of so-called globally observable games in linear partial monitoring \citep{kirschner2023linear}.

\paragraph{Application to Linear Feedback Models}

To illustrate the techniques introduced, we compute a regret bound for Algorithm \ref{alg:e2d} for  linear bandits with side-observations (\cref{ex:linear-bandits,ex:semi-bandits}).
\begin{lemma}\label{lem:dec-bounds-linear}
	For linear bandits with side-observations and divergence $I_f(\pi, g) = \|M_{\pi}(g-f)\|^2$, 
	\begin{align*}
		\acfpacdec_{\eps}(f) \leq \min_{\mu \in \sP(\Pi)} \max_{b \in \Pi} \eps \|\phi_b\|_{V(\mu)^{-1}} \leq \eps \sqrt{d}
	\end{align*}
	where $V(\mu) = \sum_{\pi \in \Pi} \mu_\pi M_{\pi} M_\pi^\T$. %In particular, for linear bandit setting, $\pacdec_{\eps}(f) \leq \eps \sqrt{d}$.
	Moreover, denoting $\Omega = \min_{\mu \in \sP(\Pi)} \max_{b \in \Pi} \|\phi_b\|_{V(\mu)^{-1}}$,
		\begin{align*}
		\acfdec_{\eps}(f) \leq \min\left(\eps \sqrt{\Psi(f)},
  % \eps^{2/3} \min_{\mu \in \sP(\Pi)} \max_{b \in \Pi} \eps \|\phi_b\|_{V(\mu)^{-1}} \right) \leq \min\left(\eps \sqrt{d}, \eps^{2/3} d^{1/3} \right)
  2\eps^{2/3} \Omega^{1/3} \Delta_{\max}^{1/3} \right)
	\end{align*}
\end{lemma}
The proof is given in \cref{app:proof-dec-bounds-linear}. While in the worst-case for linear bandits, there is no improvement over the standard $\cO(d \sqrt{n})$ without further refinement or specification of the upper bounds, in the case of linear side-observations there is an improvement whenever $\Omega \leq \max_{f \in \cM} \Psi(f)$.
% \begin{remark}
% \label{rem:lin-band-reg}
% 	For linear bandits, \cref{lem:dec-bounds-linear} implies that $\acdec_\eps (f) \leq \min \{\eps d^{1/2}, \eps^{2/3}d^{1/3} \}$. For $\Est_n \leq d \log(n)$, the implied (anytime) regret bound of \cref{alg:e2d} is $R_n \leq \tilde \cO(\min \{d \sqrt{n}, d^{2/3} n^{2/3}\})$.
% 	% \begin{align*}
% 	% 	R_n \leq \tilde \cO(\min \{d \sqrt{n}, d^{2/3} n^{2/3}\})
% 	% \end{align*}
% \end{remark}
% This rate is better than the standard $d \sqrt{n}$ rate in the regime $n \leq d^2$. 
To exemplify the improvement, consider a semi-bandit with a ``revealing'' action $\hat \pi$, e.g. $M_{\hat \pi} = \eye_d$. Here, the regret bound improves to $R_n \leq \min \{d \sqrt{n}, d^{1/3} n^{2/3}\}$, since then $\acfpacdec_\eps(f) \leq \eps$. The corresponding improvement in the regime $n \leq d^4$ might seem modest, but is relevant in high-dimensional and non-parametric models.
Moreover, in (deep) reinforcement learning, high-dimensional models are commonly used and the learner obtains side information in the form of state observations. 
Therefore, it is plausible that the $n^{2/3}$ rate is dominant even for a moderate horizon. Exploring this effect in reinforcement learning is therefore an important direction for future work. \looseness=-1

Notably, this improvement is \emph{not} observed by upper confidence bound algorithms and Thompson sampling, because both approaches discard informative but suboptimal actions early on \cite[c.f.~][]{lattimore2017end}, including the action $\hat \pi$ in the example above. E2D for a constant offset parameter $\lambda >0$, in principle, attains the better rate, but only if one pre-commits to a fixed horizon. Lastly, we note that a similar effect was observed for information-directed sampling in sparse high-dimensional linear bandits \citep{hao2020high}. \looseness=-1

\subsection{Computational Aspects} \label{sec:comp-aspects}

For finite model classes, \cref{alg:e2d} can be readily implemented. Since almost no structure is imposed on the gap and information matrices of size $|\Pi| \times |\cH|$, avoiding scaling with $|\Pi|\cdot|\cH|$ seems hardly possible without introducing additional assumptions. Even in the finite case, solving \cref{eq:d*} is not immediate because the corresponding Lagrangian is not convex-concave. A practical approach is to solve the inner saddle point for \cref{eq:L2} as a function of $\lambda$. Strong duality holds for the inner problem, and one can obtain a solution efficiently by solving the corresponding linear program using standard solvers. It then remains to optimize over $\lambda \geq 0$. This can be done, for example, via a grid search over the range $[0, \max_{f \in \cH}\eps^{-2} \acdec_{\eps}(f)]$.%\todoj{explain?}%, where we used \cref{lem:d_eps_bounds_2} to restrict the range.

In the linear setting, the above is not satisfactory because most commonly $\cM$ is identified with parameters in $\bR^d$. As noted before, ridge regression can be used instead of online aggregation while preserving the optimal scaling of the estimation error (see \cref{app:ridge}). The next lemma further shows that the saddle point problem \cref{eq:d*} can be rewritten to only scale with the size of the decision set $|\Pi|$. \looseness=-1

\begin{lemma}\label{lem:linear-dec}
	Consider linear bandits with side observations, $\cH = \bR^d$ and quadratic divergence, $I_f(\pi, g) = \|M_{\pi}(g-f)\|^2$, and denote $\phi_\mu = \sum_{\pi \in \Pi} \mu_\pi \phi_\pi$ and $V(\mu) = \sum_{\pi \in \Pi} \mu_{\pi} M_\pi^\top M_{\pi}$. Then
	\begin{align*}
		\acfdec_{\eps}(\hat f_t) = \min_{\lambda \geq 0} \min_{\mu \in \sP(\Pi)} \max_{b \in \Pi} \ip{\phi_b - \phi_\mu, \hat f_t} + \frac{1}{4\lambda} \|\phi_b\|_{V(\mu)^{-1}}^2  + \lambda \eps^2
	\end{align*}
	Moreover, the objective is convex in $\mu \in \sP(\Pi)$. 
%	In particular, 
%	\begin{align}
%		\acfdec_{\eps}(\hat f_t) = \min_{\mu \in \sP(\Pi)} \max_{\omega \in \sP(\Pi)} \sum_{\pi \in \Pi}\omega_\pi\big(\ip{\phi_b - \phi_\mu, \hat f_t} + \eps \|\phi_b\|_{V(\mu)^{-1}}\big)
%	\end{align}
\end{lemma}
The proof is a straightforward calculation provided in \cref{app:proof-lem-linear-dec}. 
Note that the saddle point expression is analogous to \cref{eq:L2}, and in fact, one can linearize the inner maximization over $\sP(\Pi)$, such that the inner saddle point becomes convex-concave. 
This leads to expressions equivalent to \cref{eq:L1,eq:L3}, albeit the objective is no longer linear in $\mu \in \sP(\Pi)$. 
We use \cref{lem:linear-dec} to employ the same strategy as before: As a function of $\lambda \geq 0$, solve the inner problem of the expression in \cref{lem:linear-dec}, for example, as a convex program with $|\Pi|$ variables and $|\Pi|$ constraints (\cref{app:lp-fixed-lambda}). Then all that remains is to solve a one-dimensional optimization problem over $\lambda  \in [0, \max_{f \in \cH}\eps^{-2} \acdec_{\eps}(f)]$. We demonstrate this approach in \cref{app:experiments} to showcase the performance of E2D on simple examples.

\section{Conclusion}

We introduced \AETD, an algorithm based on the estimation-to-decisions framework for sequential decision-making with structured observations. The algorithm optimizes the \emph{average-constrained} decision-making coefficient, which can be understood as a reparametrization of the corresponding offset version. The reparametrization facilitates an elegant anytime analysis and makes setting all remaining hyperparameters immediate. We demonstrate the improvement with a novel bound for linear bandits with side-observations, that is not attained by previous approaches. Lastly, we discuss how the algorithm can be implemented for finite and linear model classes. Nevertheless, much remains to be done. For example, one can expect the reference model to change very little from round to round, and therefore, it seems wasteful to solve \cref{eq:d*} from scratch repetitively. Preferable instead would be an incremental scheme that iteratively computes updates to the sampling distribution. \looseness=-1

\begin{ack}
Johannes Kirschner gratefully acknowledges funding from the SNSF Early Postdoc.Mobility fellowship P2EZP2\_199781.
\end{ack}

\bibliographystyle{abbrvnat}
\bibliography{refs}

\begin{thebibliography}{52}
\providecommand{\natexlab}[1]{#1}
\providecommand{\url}[1]{\texttt{#1}}
\expandafter\ifx\csname urlstyle\endcsname\relax
  \providecommand{\doi}[1]{doi: #1}\else
  \providecommand{\doi}{doi: \begingroup \urlstyle{rm}\Url}\fi

\bibitem[Abbasi-Yadkori et~al.(2011)Abbasi-Yadkori, P{\'a}l, and
  Szepesv{\'a}ri]{abbasi2011improved}
Y.~Abbasi-Yadkori, D.~P{\'a}l, and C.~Szepesv{\'a}ri.
\newblock Improved algorithms for linear stochastic bandits.
\newblock \emph{Advances in neural information processing systems}, 24, 2011.

\bibitem[Abe and Long(1999)]{abe1999associative}
N.~Abe and P.~M. Long.
\newblock Associative reinforcement learning using linear probabilistic
  concepts.
\newblock In \emph{Proceedings of the Sixteenth International Conference on
  Machine Learning}, ICML '99, pages 3--11, San Francisco, CA, USA, 1999.
  Morgan Kaufmann Publishers Inc.
\newblock ISBN 1-55860-612-2.

\bibitem[Agarwal and Zhang(2022)]{agarwal2022model}
A.~Agarwal and T.~Zhang.
\newblock Model-based rl with optimistic posterior sampling: Structural
  conditions and sample complexity.
\newblock \emph{arXiv preprint arXiv:2206.07659}, 2022.

\bibitem[Azar et~al.(2017)Azar, Osband, and Munos]{azar2017minimax}
M.~G. Azar, I.~Osband, and R.~Munos.
\newblock Minimax regret bounds for reinforcement learning.
\newblock In \emph{International Conference on Machine Learning}, pages
  263--272. PMLR, 2017.

\bibitem[Bubeck et~al.(2015)Bubeck, Dekel, Koren, and Peres]{bubeck2015bandit}
S.~Bubeck, O.~Dekel, T.~Koren, and Y.~Peres.
\newblock Bandit convex optimization:$\backslash$sqrtt regret in one dimension.
\newblock In \emph{Conference on Learning Theory}, pages 266--278. PMLR, 2015.

\bibitem[Cesa-Bianchi and Lugosi(2006)]{cesa2006prediction}
N.~Cesa-Bianchi and G.~Lugosi.
\newblock \emph{Prediction, learning, and games}.
\newblock Cambridge university press, 2006.

\bibitem[Chen et~al.(2022)Chen, Mei, and Bai]{chen2022unified}
F.~Chen, S.~Mei, and Y.~Bai.
\newblock Unified algorithms for rl with decision-estimation coefficients:
  No-regret, pac, and reward-free learning.
\newblock \emph{arXiv preprint arXiv:2209.11745}, 2022.

\bibitem[Combes et~al.(2015)Combes, Magureanu, Proutiere, and
  Laroche]{combes15learning}
R.~Combes, S.~Magureanu, A.~Proutiere, and C.~Laroche.
\newblock Learning to rank: Regret lower bounds and efficient algorithms.
\newblock In \emph{Proceedings of the 2015 ACM SIGMETRICS International
  Conference on Measurement and Modeling of Computer Systems}, pages 231--244.
  ACM, 2015.
\newblock ISBN 978-1-4503-3486-0.

\bibitem[Combes et~al.(2017)Combes, Magureanu, and
  Proutiere]{combes2017minimal}
R.~Combes, S.~Magureanu, and A.~Proutiere.
\newblock Minimal exploration in structured stochastic bandits.
\newblock \emph{Advances in Neural Information Processing Systems}, 30, 2017.

\bibitem[Dann et~al.(2021)Dann, Mohri, Zhang, and Zimmert]{dann2021provably}
C.~Dann, M.~Mohri, T.~Zhang, and J.~Zimmert.
\newblock A provably efficient model-free posterior sampling method for
  episodic reinforcement learning.
\newblock \emph{Advances in Neural Information Processing Systems},
  34:\penalty0 12040--12051, 2021.

\bibitem[Degenne et~al.(2020{\natexlab{a}})Degenne, M{\'e}nard, Shang, and
  Valko]{degenne2020gamification}
R.~Degenne, P.~M{\'e}nard, X.~Shang, and M.~Valko.
\newblock Gamification of pure exploration for linear bandits.
\newblock In \emph{International Conference on Machine Learning}, pages
  2432--2442. PMLR, 2020{\natexlab{a}}.

\bibitem[Degenne et~al.(2020{\natexlab{b}})Degenne, Shao, and
  Koolen]{degenne2020games}
R.~Degenne, H.~Shao, and W.~Koolen.
\newblock Structure adaptive algorithms for stochastic bandits.
\newblock In H.~D. III and A.~Singh, editors, \emph{Proceedings of the 37th
  International Conference on Machine Learning}, volume 119 of
  \emph{Proceedings of Machine Learning Research}, pages 2443--2452. PMLR,
  13--18 Jul 2020{\natexlab{b}}.
\newblock URL \url{https://proceedings.mlr.press/v119/degenne20b.html}.

\bibitem[den Boer(2015)]{denBoer15}
A.~V. den Boer.
\newblock Dynamic pricing and learning: Historical origins, current research,
  and new directions.
\newblock \emph{Surveys in Operations Research and Management Science},
  20\penalty0 (1):\penalty0 1--18, 2015.

\bibitem[Dong and Ma(2022)]{dong2022asymptotic}
K.~Dong and T.~Ma.
\newblock Asymptotic instance-optimal algorithms for interactive decision
  making.
\newblock \emph{arXiv preprint arXiv:2206.02326}, 2022.

\bibitem[Du et~al.(2021)Du, Kakade, Lee, Lovett, Mahajan, Sun, and
  Wang]{du2021bilinear}
S.~Du, S.~Kakade, J.~Lee, S.~Lovett, G.~Mahajan, W.~Sun, and R.~Wang.
\newblock Bilinear classes: A structural framework for provable generalization
  in rl.
\newblock In \emph{International Conference on Machine Learning}, pages
  2826--2836. PMLR, 2021.

\bibitem[Dunn and Harshbarger(1978)]{dunn1978conditional}
J.~C. Dunn and S.~Harshbarger.
\newblock Conditional gradient algorithms with open loop step size rules.
\newblock \emph{Journal of Mathematical Analysis and Applications}, 62\penalty0
  (2):\penalty0 432--444, 1978.

\bibitem[Foster et~al.(2021)Foster, Kakade, Qian, and
  Rakhlin]{foster2021statistical}
D.~J. Foster, S.~M. Kakade, J.~Qian, and A.~Rakhlin.
\newblock The statistical complexity of interactive decision making.
\newblock \emph{arXiv preprint arXiv:2112.13487}, 2021.

\bibitem[Foster et~al.(2022{\natexlab{a}})Foster, Golowich, Qian, Rakhlin, and
  Sekhari]{foster2022note}
D.~J. Foster, N.~Golowich, J.~Qian, A.~Rakhlin, and A.~Sekhari.
\newblock A note on model-free reinforcement learning with the
  decision-estimation coefficient.
\newblock \emph{arXiv preprint arXiv:2211.14250}, 2022{\natexlab{a}}.

\bibitem[Foster et~al.(2022{\natexlab{b}})Foster, Rakhlin, Sekhari, and
  Sridharan]{foster2022complexity}
D.~J. Foster, A.~Rakhlin, A.~Sekhari, and K.~Sridharan.
\newblock On the complexity of adversarial decision making.
\newblock \emph{arXiv preprint arXiv:2206.13063}, 2022{\natexlab{b}}.

\bibitem[Foster et~al.(2023)Foster, Golowich, and Han]{foster2023tight}
D.~J. Foster, N.~Golowich, and Y.~Han.
\newblock Tight guarantees for interactive decision making with the
  decision-estimation coefficient.
\newblock \emph{arXiv preprint arXiv:2301.08215}, 2023.

\bibitem[Frank and Wolfe(1956)]{frank1956algorithm}
M.~Frank and P.~Wolfe.
\newblock An algorithm for quadratic programming.
\newblock \emph{Naval research logistics quarterly}, 3\penalty0 (1-2):\penalty0
  95--110, 1956.

\bibitem[Garcelon et~al.(2020)Garcelon, Roziere, Meunier, Tarbouriech, Teytaud,
  Lazaric, and Pirotta]{garcelon2020adversarial}
E.~Garcelon, B.~Roziere, L.~Meunier, J.~Tarbouriech, O.~Teytaud, A.~Lazaric,
  and M.~Pirotta.
\newblock Adversarial attacks on linear contextual bandits.
\newblock \emph{Advances in Neural Information Processing Systems},
  33:\penalty0 14362--14373, 2020.

\bibitem[Gy{\"o}rgy et~al.(2013)Gy{\"o}rgy, P{\'a}l, and
  Szepesv{\'a}ri]{gyorgy2013online}
A.~Gy{\"o}rgy, D.~P{\'a}l, and C.~Szepesv{\'a}ri.
\newblock Online learning: Algorithms for big data.
\newblock \emph{Lecture Notes}, 2013.

\bibitem[Hao et~al.(2020)Hao, Lattimore, and Wang]{hao2020high}
B.~Hao, T.~Lattimore, and M.~Wang.
\newblock High-dimensional sparse linear bandits.
\newblock \emph{Advances in Neural Information Processing Systems},
  33:\penalty0 10753--10763, 2020.

\bibitem[Jaggi(2013)]{jaggi2013revisiting}
M.~Jaggi.
\newblock Revisiting frank-wolfe: Projection-free sparse convex optimization.
\newblock In \emph{International conference on machine learning}, pages
  427--435. PMLR, 2013.

\bibitem[Jin et~al.(2018)Jin, Allen-Zhu, Bubeck, and Jordan]{jin2018q}
C.~Jin, Z.~Allen-Zhu, S.~Bubeck, and M.~I. Jordan.
\newblock Is q-learning provably efficient?
\newblock \emph{Advances in neural information processing systems}, 31, 2018.

\bibitem[Kirschner and Krause(2018)]{kirschner2018information}
J.~Kirschner and A.~Krause.
\newblock Information directed sampling and bandits with heteroscedastic noise.
\newblock In \emph{Conference On Learning Theory}, pages 358--384. PMLR, 2018.

\bibitem[Kirschner and Krause(2021)]{kirschner2021dueling}
J.~Kirschner and A.~Krause.
\newblock Bias-robust bayesian optimization via dueling bandits.
\newblock In \emph{International Conference on Machine Learning}, pages
  5595--5605. PMLR, 2021.

\bibitem[Kirschner et~al.(2020{\natexlab{a}})Kirschner, Bogunovic, Jegelka, and
  Krause]{kirschner2020distributionally}
J.~Kirschner, I.~Bogunovic, S.~Jegelka, and A.~Krause.
\newblock {Distributionally Robust Bayesian Optimization}.
\newblock In \emph{Proc. International Conference on Artificial Intelligence
  and Statistics (AISTATS)}, August 2020{\natexlab{a}}.
\newblock URL
  \url{http://proceedings.mlr.press/v108/kirschner20a/kirschner20a.pdf}.

\bibitem[Kirschner et~al.(2020{\natexlab{b}})Kirschner, Lattimore, and
  Krause]{kirschner2020linearpm}
J.~Kirschner, T.~Lattimore, and A.~Krause.
\newblock Information directed sampling for linear partial monitoring.
\newblock In \emph{Conference on Learning Theory}, pages 2328--2369. PMLR,
  2020{\natexlab{b}}.

\bibitem[Kirschner et~al.(2021)Kirschner, Lattimore, Vernade, and
  Szepesv{\'a}ri]{kirschner2021asymptotically}
J.~Kirschner, T.~Lattimore, C.~Vernade, and C.~Szepesv{\'a}ri.
\newblock Asymptotically optimal information-directed sampling.
\newblock In \emph{Conference on Learning Theory}, pages 2777--2821. PMLR,
  2021.

\bibitem[Kirschner et~al.(2023)Kirschner, Lattimore, and
  Krause]{kirschner2023linear}
J.~Kirschner, T.~Lattimore, and A.~Krause.
\newblock Linear partial monitoring for sequential decision making: Algorithms,
  regret bounds and applications.
\newblock \emph{Journal of Machine Learning Research}, August 2023.

\bibitem[Lattimore and Gyorgy(2021)]{lattimore2021mirror}
T.~Lattimore and A.~Gyorgy.
\newblock Mirror descent and the information ratio.
\newblock In \emph{Conference on Learning Theory}, pages 2965--2992. PMLR,
  2021.

\bibitem[Lattimore and Szepesvari(2017)]{lattimore2017end}
T.~Lattimore and C.~Szepesvari.
\newblock The end of optimism? an asymptotic analysis of finite-armed linear
  bandits.
\newblock In \emph{Artificial Intelligence and Statistics}, pages 728--737.
  PMLR, 2017.

\bibitem[Lattimore and Szepesv{\'a}ri(2019)]{lattimore2019information}
T.~Lattimore and C.~Szepesv{\'a}ri.
\newblock An information-theoretic approach to minimax regret in partial
  monitoring.
\newblock In \emph{Conference on Learning Theory}, pages 2111--2139. PMLR,
  2019.

\bibitem[Lattimore and Szepesv{\'a}ri(2020{\natexlab{a}})]{lattimore2020bandit}
T.~Lattimore and C.~Szepesv{\'a}ri.
\newblock \emph{Bandit algorithms}.
\newblock Cambridge University Press, 2020{\natexlab{a}}.

\bibitem[Lattimore and
  Szepesv{\'a}ri(2020{\natexlab{b}})]{lattimore2020exploration}
T.~Lattimore and C.~Szepesv{\'a}ri.
\newblock Exploration by optimisation in partial monitoring.
\newblock In \emph{Conference on Learning Theory}, pages 2488--2515. PMLR,
  2020{\natexlab{b}}.

\bibitem[Lattimore et~al.(2018)Lattimore, Kveton, Li, and
  Szepesv{\'a}ri]{LKLS18}
T.~Lattimore, B.~Kveton, S.~Li, and C.~Szepesv{\'a}ri.
\newblock Toprank: A practical algorithm for online stochastic ranking.
\newblock In \emph{Advances in Neural Information Processing Systems}, pages
  3949--3958. Curran Associates, Inc., 2018.

\bibitem[Lin et~al.(2014)Lin, Abrahao, Kleinberg, Lui, and
  Chen]{lin2014combinatorial}
T.~Lin, B.~Abrahao, R.~Kleinberg, J.~Lui, and W.~Chen.
\newblock Combinatorial partial monitoring game with linear feedback and its
  applications.
\newblock In \emph{International Conference on Machine Learning}, pages
  901--909, 2014.

\bibitem[Mannor and Shamir(2011)]{mannor2011bandits}
S.~Mannor and O.~Shamir.
\newblock From bandits to experts: On the value of side-observations.
\newblock \emph{Advances in Neural Information Processing Systems}, 24, 2011.

\bibitem[Osband et~al.(2016)Osband, Blundell, Pritzel, and
  Van~Roy]{osband2016deep}
I.~Osband, C.~Blundell, A.~Pritzel, and B.~Van~Roy.
\newblock Deep exploration via bootstrapped dqn.
\newblock \emph{Advances in neural information processing systems}, 29, 2016.

\bibitem[Radlinski et~al.(2008)Radlinski, Kleinberg, and
  Joachims]{radlinski08learning}
F.~Radlinski, R.~Kleinberg, and T.~Joachims.
\newblock Learning diverse rankings with multi-armed bandits.
\newblock In \emph{Proceedings of the 25th International Conference on Machine
  Learning}, pages 784--791. ACM, 2008.

\bibitem[Russo and Van~Roy(2014)]{russo2014learning}
D.~Russo and B.~Van~Roy.
\newblock Learning to optimize via information-directed sampling.
\newblock \emph{Advances in Neural Information Processing Systems}, 27, 2014.

\bibitem[Russo and Van~Roy(2016)]{russo2016information}
D.~Russo and B.~Van~Roy.
\newblock An information-theoretic analysis of thompson sampling.
\newblock \emph{The Journal of Machine Learning Research}, 17\penalty0
  (1):\penalty0 2442--2471, 2016.

\bibitem[Rustichini(1999)]{rustichini1999minimizing}
A.~Rustichini.
\newblock Minimizing regret: The general case.
\newblock \emph{Games and Economic Behavior}, 29\penalty0 (1-2):\penalty0
  224--243, 1999.

\bibitem[Shalev-Shwartz et~al.(2012)]{shalev2012online}
S.~Shalev-Shwartz et~al.
\newblock Online learning and online convex optimization.
\newblock \emph{Foundations and Trends{\textregistered} in Machine Learning},
  4\penalty0 (2):\penalty0 107--194, 2012.

\bibitem[Srinivas et~al.(2010)Srinivas, Krause, Kakade, and
  Seeger]{srinivas10gaussian}
N.~Srinivas, A.~Krause, S.~Kakade, and M.~Seeger.
\newblock Gaussian process optimization in the bandit setting: No regret and
  experimental design.
\newblock In \emph{Proc. International Conference on Machine Learning (ICML)},
  2010.

\bibitem[Thompson(1933)]{thompson1933likelihood}
W.~R. Thompson.
\newblock On the likelihood that one unknown probability exceeds another in
  view of the evidence of two samples.
\newblock \emph{Biometrika}, 25\penalty0 (3-4):\penalty0 285--294, 1933.

\bibitem[Yue and Joachims(2009)]{YJ09}
Y.~Yue and T.~Joachims.
\newblock Interactively optimizing information retrieval systems as a dueling
  bandits problem.
\newblock In \emph{Proceedings of the 26th International Conference on Machine
  Learning}, pages 1201--1208. ACM, 2009.

\bibitem[Zanette et~al.(2020)Zanette, Lazaric, Kochenderfer, and
  Brunskill]{zanette2020learning}
A.~Zanette, A.~Lazaric, M.~Kochenderfer, and E.~Brunskill.
\newblock Learning near optimal policies with low inherent bellman error.
\newblock In \emph{International Conference on Machine Learning}, pages
  10978--10989. PMLR, 2020.

\bibitem[Zhang(2022)]{zhang2022feel}
T.~Zhang.
\newblock Feel-good thompson sampling for contextual bandits and reinforcement
  learning.
\newblock \emph{SIAM Journal on Mathematics of Data Science}, 4\penalty0
  (2):\penalty0 834--857, 2022.

\bibitem[Zhou et~al.(2021)Zhou, Gu, and Szepesvari]{zhou2021nearly}
D.~Zhou, Q.~Gu, and C.~Szepesvari.
\newblock Nearly minimax optimal reinforcement learning for linear mixture
  markov decision processes.
\newblock In \emph{Conference on Learning Theory}, pages 4532--4576. PMLR,
  2021.

\end{thebibliography}

\newpage

\appendix

\section{Online Density Estimation}\label{app:ewa}

\begin{algorithm}[t]
	\DontPrintSemicolon
	\SetAlgoVlined
	\SetAlgoNoLine
	\SetAlgoNoEnd
	
	\caption{Exponential Weights Algorithm (EWA) for Density Estimation} \label{alg:ewa}
	\Input{Finite model class $\cM$, data $\cD_{t} = \{(y_1, \pi_1), \dots, (y_t, \pi_t)\}$, Learning rate $\eta > 0$}
	Define $L(f) = -\sum_{s=1}^t \log M_f(y_s|\pi_s)$\;
	Let $p(f) \propto \exp(-\eta L(f))$\;
	For convex $\cH$: Return $\sum_{f \in \cH} p(f) f$\;
	Else: Return $f \sim p(\cdot)$.
\end{algorithm}

For any $f \in \cH$ and $\pi \in \Pi$, we denote by $p(\cdot|\pi, f)$ the the density function of the observation distribution $M_{f}(\pi)$ w.r.t.~a reference measure over the observation space $\cO$. Consider a finite model class $\cH$ and the KL divergence,
\begin{align}
	 e_{\pi} I_{f} e_{g} &= \EE_{y \sim M_{g}(\pi)}[\log \left(\frac{p(y|\pi, g)}{p(y|\pi, f)}\right)]
\end{align}
In this case, the estimation error can be written as follows:
 \begin{align*}
	\Est_n = \EE[\sum_{t=1}^n e_{\pi_t} I_{\hat f_t} e_{f^*}] &= \EE[\sum_{t=1}^n \log(p(y_t|\pi_t, f^*)/ p(y_t|\pi_t, \hat f_t)]\\
	&=\EE[\sum_{t=1}^n \log\left(\frac{1}{p(y_t|\pi_t, \hat f_t)}\right) - \sum_{t=1}^n \log\left(\frac{1}{p(y_t|\pi_t, f^*)}\right)]
%	&\leq\EE[\sum_{t=1}^n \log\left(\frac{1}{p(y|\pi_t, \hat f_t)}\right) - \inf_{g \in \cH} \sum_{t=1}^n \log\left(\frac{1}{p(y|\pi_t, g)}\right)] \leq \log(|\cH|)
\end{align*}
The last line can be understood as the \emph{estimation regret} of the estimates $\hat f_1, \dots, \hat f_n$ under the logarithmic loss. A classical approach to control this term is the \emph{exponential weights algorithm} (EWA) given in \cref{alg:ewa}. For the EWA algorithm, we have the following bound.
\begin{lemma}[EWA for Online Density Estimation]
	For any data stream $\{y_1, \pi_1, \dots, y_n, \pi_n \}$ the predictions $\hat f_1, \dots \hat f_n$ obtained via 
	\cref{alg:ewa} with $\eta = 1$ satisfy
	\begin{align}
		\Est_n \leq \EE[\sum_{t=1}^n \log\left(\frac{1}{p(y_t|\pi_t, \hat f_t)}\right) - \inf_{g \in \cH} \sum_{t=1}^n \log\left(\frac{1}{p(y_t|\pi_t, g)}\right)] \leq \log(|\cH|)
	\end{align}
%	where the expectation is over the randomness of the EWA algorithm.
\end{lemma}
For a proof, see \cite[Proposition 3.1]{cesa2006prediction}.

\subsection{Bounding the Estimation Error of Projected Regularized Least-Squares}\label{app:ridge}
In this section, we consider the linear model from \cref{ex:semi-bandits}. We denote by $\|\cdot\|$ the Euclidean norm. For simplicity, the observation maps $M_\pi \in \bR^{m \times d}$ are assumed to have the same output dimension $m \in \bN$. The observation distribution is such that $y_t =  M_{\pi_t}f^* + \xi_t$, where $\xi \in \bR^m$ is random noise such that $\EE_t[\xi] = 0$ and $\EE_t[\|\xi\|^2] \leq \sigma^2$. Here, $\EE_t[\cdot] = \EE[\cdot|\pi_1,y_1, \dots, \pi_{t-1}, y_{t-1}, \pi_t]$ is the conditional observation in round $t$ including the decision $\pi_t$ chosen in round $t$.

We will use the quadratic divergence\footnote{We added a factor of $\frac{1}{2}$ for convenience.}, $e_\pi I_f e_g = \frac{1}{2}\|M_\pi (g-f)\|^2$ This choice corresponds to the Gaussian KL, but we do not require that the noise distribution is Gaussian is the following. In the linear bandit model, this choice reduces to $e_\pi I_f e_g = \frac{1}{2}\ip{\phi_\pi, g-f}^2$.

Let $K \subset \bR^d$ be a closed convex set. Our goal is to control the estimation regret for the projected regularized least-squares estimator,
\begin{align}
	\hat f_t = \argmin_{f \in K} \sum_{s=1}^{t-1} \|M_{\pi_s} f- y_s\|^2 + \|f\|_{V_0}^2 = \text{Proj}_{V_t}\left(V_t^{-1} \sum_{s=1}^{t-1} M_{\pi_s}^\top y_s\right) \label{eq:ridge}
\end{align}
where $V_0$ is a positive definite matrix, $V_t = \sum_{s=1}^{t-1} M_{\pi_s}^\top M_{\pi_s} + V_0$ and $\text{Proj}_{V_t}(\cdot)$ is the orthogonal projection w.r.t.~the $\|\cdot\|_{V_t}$ norm. For $K=\bR^d$ and $V_0 = \eta \eye_d$, this recovers the standard ridge regression. The projection is necessary to bound the magnitude of the squared loss, and the result will depend on an almost-surely bound on the `observed' diameter, $$\max_{f,g \in K} \max_{\pi \in \Pi} \|M_\pi (f-g)\| \leq B$$

%The divergence is set to the Gaussian KL,
%\begin{align*}
%I_f(\pi, g) = \frac{1}{2} \ip{\phi_\pi, g - f}^2
%\end{align*}
Recall that our goal is to bound the estimation error,
\begin{align}
	\Est_n &= \EE[\sum_{t=1}^n e_{\pi_t} I_{\hat f_t} e_{f^*}] = \EE[ \sum_{t=1}^n \tfrac{1}{2} \|M_{\pi_t}  (f^* - \hat f_t)\|^2]
\end{align}
We remark that one can get the following naive bound by applying Cauchy-Schwarz:
\begin{align}
	\sum_{t=1}^n \|M_{\pi_t}  (f^* - \hat f_t)\|^2 \leq \sum_{t=1}^n \|M_{\pi_t}\|_{V_t^{-1}}^2\|f^* - \hat f_t\|_{V_t}^2 \leq \cO(d^2 \log(n)^2)
\end{align}
The last inequality follows from the elliptic potential lemma and standard concentration inequalities \cite[Lemma 19.4 and Theorem 20.5]{lattimore2020bandit}. However, this will lead to an additional $d$-factor in the regret that can be avoided, as we see next.

For $K = \bR^d$, one-dimensional observations and noise bounded in the range $[-\bar B, \bar B]$, one can also directly apply \cite[Theorem 11.7]{cesa2006prediction} to get $\Est_n \leq \cO(\bar B^2 d \log(n))$, thereby improving the naive bound by a factor $d \log(n)$. This result is obtained in a more general setting, where no assumptions, other than boundedness, are placed on the observation sequence $y_1,\dots,y_n$. Here we refine and generalize this result in two directions: First, we allow for the more general feedback model in with multi-dimensional observations (\cref{ex:semi-bandits}). Second, we directly exploit the stochastic observation model to obtain a stronger result that does not require the observation noise to be bounded.

\begin{theorem}\label{thm:projected-ridge-error}
	Consider the linear observation setting with additive noise and quadratic divergence $e_\pi I_f e_g = \frac{1}{2} \|M_\pi (g-f)\|^2$, as described at the beginning of this section. Assume that $\max_{f,g \in \cH, \pi \in \Pi}\|M_\pi (f-g)\| \leq B$ and $\EE[\|\xi_t\|^2] \leq \sigma^2$.
%	 Let $\hat f_t$ be the projected ridge estimator (\cref{eq:ridge}), observations $y_t = M_{a_t} f + \eps_t$ with zero-mean, conditionally independent noise $\eps_t \in \bR^m$, and the quadratic divergence $e_\pi I_f e_g = \|M_{\pi} (g- f)\|^2$. Assume that $\|M_a (f-g)\| \leq B$ and $\EE[\|\eps_t\|^2] \leq \sigma^2$. 
	Then
	\begin{align*}
		\Est_n \leq (\sigma^2 + B^2) \EE[\log \left(\frac{\det V_n}{\det V_0}\right)]
	\end{align*}
	If in addition $\|M_\pi\| \leq L$ and $V_0 = \eta \eye_d$, then $\Est_n \leq (\sigma^2 + B^2)\log\big(1 + \frac{nL^2}{\eta d}\big)$. %\todoj{is this spectral norm or Frobenius norm?}
\end{theorem}
\begin{remark}
    Note that by the \cite[Theorem 19.4]{lattimore2020bandit}, $\log \left(\frac{\det V_n}{\det V_0}\right)$ can further be upper bounded by $d \log\big(\frac{\text{trace} V_0 + n L^2}{d \det(V_0)^{1/d}}\big)$, which effectively results the desired bound.
\end{remark}

\begin{proof}
	
	The proof adapts \cite[Theorem 19.8]{gyorgy2013online} to multi-dimensional observations and takes advantage of the stochastic loss function by taking the expectation.
	
	First, define $l_t(f) = \frac{1}{2}\|M_{\pi_t} f - y_t\|^2$.  Then, using that $\EE_t[y_t] = M_{\pi_t}f^*$,
	\begin{align*}
		\Est_n = \EE[\sum_{t=1}^n \tfrac{1}{2}\|M_{\pi_t}  (f^* - \hat f_t)\|^2] &= \EE[\sum_{t=1}^n \tfrac{1}{2} \|M_{\pi_t}  \hat f_t - y_t\|^2  -  \tfrac{1}{2}\|M_{\pi_t}  f^* - y_t\|^2] \\
		&= \EE[\sum_{t=1}^n l_t(\hat f_t) - l_t(f^*)]
	\end{align*}
	Further, by directly generalizing \cite[Lemma 19.7]{gyorgy2013online}, we have that
	\begin{align}
		l_t(\hat f_{t+1})- l_t(\hat f_t) \leq \nabla l_t(\hat f_t) V_t^{-1} \nabla l_t(\hat f_t)
		&= (M_{\pi_t} \hat f_t - y_t)^\top M_{\pi_t}^\top V_t^{-1} M_{\pi_t} (M_{\pi_t} \hat f_t - y_t) \label{eq:loss difference bound}
	\end{align}

	We now start upper bounding the estimation error,
	\begin{align}
		\Est_n &\stackrel{(i)}{\leq} \|f^*\|^2 +  \EE[\sum_{t=1}^n \big(l_t(w_t) -l_t(w_{t+1})\big)] \nonumber \\
		&\stackrel{(ii)}{\leq} \|f^*\|^2 + \EE[\sum_{t=1}^n \big(\xi_t  + M_{\pi_t} (f^* - \hat  f_{t-1})\big) M_{\pi_t} V_t^{-1} M_{\pi_t} \big(\xi_t  + M_{\pi_t} (f^* - \hat f_{t-1})\big)]\nonumber \\
		&\stackrel{(iii)}{=} \|f^*\|^2 + \EE[\sum_{t=1}^n \xi_t M_{\pi_t} V_t^{-1} M_{\pi_t} \xi_t \big)] + \EE[\sum_{t=1}^n \bar x_t M_{\pi_t} V_t^{-1} M_{\pi_t} \bar x_t \big)]\nonumber \\
		&\stackrel{(iv)}{\leq} \|f^*\|^2 + \EE[\sum_{t=1}^n \lambda_{\max}(M_{\pi_t} V_t^{-1} M_{\pi_t}) \| \xi_t \|^2 ] + \EE[\sum_{t=1}^n  \lambda_{\max}(M_{\pi_t} V_t^{-1} M_{\pi_t}) \| \bar x_t\|^2 \big)]\nonumber \\
		&\stackrel{(v)}{\leq }\|f^*\|^2 + (\sigma^2 + B^2) \EE[\sum_{t=1}^n  \lambda_{\max}(M_{\pi_t} V_t^{-1} M_{\pi_t})] \label{eq:l_max_upper_1}
	\end{align}
	The inequality $(i)$ follows from \cite[Lemma 2.3]{shalev2012online}. For $(ii)$ we used \cref{eq:loss difference bound}.	For $(iii)$ we used that $\EE_t[\xi_t]=0$. In $(iv)$, we introduce the maximum eigenvalue $\lambda_{\max}(A)$ for $A \in \bR^{m \times m}$ and denote $\bar x_t = M_{\pi_t} (f^* - f_{t-1})$. Lastly, in $(v)$ we used that $\|\bar x_t\|^2 \leq B$ and $\EE_t[\|\xi_t\|^2] \leq \sigma^2$.
	
	We conclude the proof with basic linear algebra. Denote by $\lambda_i(A)$ the $i$-th eigenvalue of a matrix $M \in \bR^{m \times m}$. Using the generalized matrix determinant lemma, we get
	\begin{align*}
		\det(V_{t-1}) &= \det(V_t - M_{\pi_t}^\top M_{\pi_t})\\
		&= \det(V_t)\det(I - M_{\pi_t}^\top V_t^{-1}M_{\pi_t})\\
%		&= \det(V_t)\prod_{i=1}^m (1 - \lambda_i(V_{t}^{-1/2}M_{\pi_t}^\top M_{\pi_t}V_{t}^{-1/2}))\\
		&= \det(V_t)\prod_{i=1}^m (1 - \lambda_i(M_{\pi_t}^\top V_t^{-1} M_{\pi_t}))
	\end{align*}
	Note that $\lambda_i(M_{\pi_t}^\top V_t^{-1} M_{\pi_t}) \in (0,1]$. Next, using that $\log(1 - x) \leq -x$ for all $x < 1$, we get that	
	\begin{align*}
		\log\left(\frac{\det(V_{t-1})}{\det(V_t)}\right) = \sum_{i=1}^m \log(1 -\lambda_i(M_{\pi_t}^\top V_t^{-1} M_{\pi_t})) 
		\leq - \sum_{i=1}^m \lambda_i(M_{\pi_t}^\top V_t^{-1} M_{\pi_t})
	\end{align*}
	Rearranging the last display, and bounding the sum by its maximum element, we get
	\begin{align}
		\lambda_{\max}(M_{\pi_t}^\top V_t^{-1} M_{\pi_t}) \leq
		\sum_{i=1}^m\lambda_i(M_{\pi_t}^\top V_t^{-1} M_{\pi_t})  \leq \log\left(\frac{\det(V_t)}{\det(V_{t-1})}\right)  \label{eq:l_max_upper_2}
	\end{align}
	The proof is concluded by combining \cref{eq:l_max_upper_1,eq:l_max_upper_2}.
\end{proof}

\begin{remark}[Expected Regret]
	The beauty of \cref{thm:worst-case} is that the proof uses \emph{only} in-expectation arguments. This is unlike most previous analysis, that controls the regret via controlling tail-events, and bounds on the expected regret are then derived a-posteriori from high-probability bounds. In the context of linear bandits, \cref{thm:projected-ridge-error} leads to bound on the expected regret that only requires the noise variance to be bounded, whereas most previous work relies on the stronger sub-Gaussian noise assumption \cite[e.g.][]{abbasi2011improved}.
\end{remark}

%\begin{remark}
%	Vovk-Azoury-Warmuth forecaster \cite{vovk1997competitive}
%\end{remark}

\begin{remark}[Kernel Bandits / Bayesian Optimization] Using the standard `kernel-trick', the analysis can further be extended to the non-parametric setting where $\cH$ is an infinite-dimensional reproducing kernel Hilbert space (RKHS).
\end{remark}
\section{PAC to Regret Bounds}
\label{app:pac-2-regret}
% We define a PAC version of $\acdec_\eps(f, \Delta_f)$ as follows

\subsection{Proof of \cref{lem:p-dec bound}} \label{app:proof-pac-to-regret}
%\begin{lemma}
%    Assume $\Delta = \Delta_f$ in $\pacdec_\eps^a$. 
%    Then
%    \begin{align}
%        \rdec_\eps^a(f) \le \min_{0 \geq p \geq 1} \pacdec^a_{\frac{\eps}{\sqrt{p}}}(f) + p \Delta_{\max}
%    \end{align}
%\end{lemma}
\begin{proof}
The Lagrangian for \cref{eq:pac-dec} is
% \begin{align*}
%     \acfpacdec_\eps(f) = \min_{\lambda \geq 0} \min_{\mu_1, \mu_2} \max_{\nu} \mu_1 \Delta_f \nu - \lambda (\mu_2 I_f \nu - \eps^2).
% \end{align*}
% For any $\nu$, we have $\mu_1 = e_{\pi^*_f}$, and further we can write the Lagrangian as:
\begin{align*}
    \acfpacdec_\eps(f) = \min_{\lambda \geq 0} \min_{\mu \in \sP(\Pi)} \max_{\nu \in \sP(\cM)} \delta_f \nu - \lambda (\mu I_f \nu - \eps^2).
\end{align*}
% Assume the solutions to \cref{eq:pac-dec} to be $\mu^*_2$ and $\mu^*_1 = e_{\pi^*_f}$, and 
Reparametrize any $\mu \in \sP(\Pi)$ as $\bar{\mu}(p) = (1-p) e_{\pi^*_f} + p \mu_2$.
We bound $\acfdec_\eps$ by a function of $\acfpacdec_\eps$. 
Starting from \cref{eq:L2}, we have
\begin{align*}
    \acfdec_\eps(f) 
    &= \min_{\lambda \geq 0} \min_{\mu \in \sP(\Pi)} \max_{\nu \in \sP(\cH)} \mu \Delta_f \nu - \lambda (\mu I_{f} \nu - \epsilon^2)\\
    &= \min_{\lambda \geq 0} \min_{\bar \mu \in \sP(\Pi)} \max_{\nu \in \sP(\cH)} \bar \mu \Delta_f \nu - \lambda (\bar \mu I_{f} \nu - \epsilon^2)\\
    &= \min_{\lambda \geq 0}\min_{0 \leq p \leq 1} \min_{\mu_2 \in \sP(\Pi)} \max_{\nu \in \sP(\cH)} \delta_f \nu + p \mu_2 \Delta_f e_f - \lambda \bar \mu I_{f} \nu - \lambda \epsilon^2 \\
    &\leq \min_{\lambda \geq 0}\min_{0 \leq p \leq 1} \min_{\mu_2 \in \sP(\Pi)} \max_{\nu \in \sP(\cH)} \delta_f \nu + p \mu_2 \Delta_f e_f - \lambda p \mu_2 I_{f} \nu - \lambda \epsilon^2 \\
    &\leq \min_{0 \leq p \leq 1} \min_{\lambda^\prime \geq 0} \min_{\mu_2 \in \sP(\Pi)} \max_{\nu \in \sP(\cH)} \delta_f \nu - \lambda^\prime(\mu_2 I_{f} \nu - \frac{\epsilon^2}{p} ) + p \Delta_{\max} \\
    &\leq \min_{0 \leq p \leq 1} \acfpacdec_{\frac{\eps}{\sqrt{p}}}(f) + p \Delta_{\max}\,.
\end{align*}
\end{proof}

\subsection{Proof of \cref{lem:dec-bounds-linear}}\label{app:proof-dec-bounds-linear}

\begin{proof}[Proof of \cref{lem:dec-bounds-linear}]
	For the first part, note that
	\begin{align*}
		\acfpacdec_{\eps}(f) &= \min_{\mu \in \sP(\Pi)} \min _{\lambda \geq 0} \max_{\nu \in \sP(\cH)} \delta_f \nu - \lambda \mu I_f \nu +\lambda \eps^2\\
		&=\min_{\mu \in \sP(\Pi)} \min_{\lambda \geq 0} \max_{b \in \Pi} \max_{g \in \cH} \ip{\phi_b, g} - \ip{\phi_{\pi^*_f}, f} - \lambda \|g - f\|_{V(\mu)}^2 + \lambda \eps^2\\
		&\stackrel{(i)}{=}\min_{\mu \in \sP(\Pi)} \min_{\lambda \geq 0} \max_{b \in \Pi} \ip{\phi_b - \phi_{\pi^*_f}, f} + \frac{1}{4\lambda} \|\phi_b\|_{V(\mu)^{-1}}^2 + \lambda \eps^2\\
		&\stackrel{(ii)}{\leq}\min_{\mu \in \sP(\Pi)} \min_{\lambda \geq 0} \max_{b \in \Pi} \frac{1}{4\lambda} \|\phi_b\|_{V(\mu)^{-1}}^2 + \lambda \eps^2\\
		&= \min_{\mu \in \sP(\Pi)} \max_{b \in \Pi} \eps \|\phi_b\|_{V(\mu)^{-1}}\\
		&\stackrel{(iii)}{\leq} \eps \sqrt{d}\,.
	\end{align*}
	Equation $(i)$ follows by computing the maximizer attaining the quadratic form over $\cM = \bR^d$. The inequality $(ii)$ is by definition of $\pi^*_f$ and the last inequality $(iii)$ by the assumption that the reward is observed, respectively, $\phi_\pi \phi_\pi^\top \preceq M_\pi^\T M_\pi$, and the Kiefer–Wolfowitz theorem.
	
	The second part of the statement follows by combining \cref{lem:d*-to-dc,lem:p-dec bound}.
\end{proof}

\section{Coefficient Relations Results and Proofs}
\label{app:coef-relations}

\begin{lemma}
\label{lem:data-processing-with-envs}
	Assume \cref{asm:reward-data-processing} holds, i.e. 
	\begin{align}
		(\mu (r_f-r_g))^2 \leq \mu I_f e_g\,. \label{eq:data-processing-vector}
	\end{align}
	Then
	\begin{align*}
		\left (\sum_g \mu (r_f-r_g) \nu_g\right)^2 \leq \sum_g \big(\mu (r_f-r_g)\big)^2 \nu_g \leq \mu I_f \nu\,.
	\end{align*}
\end{lemma}
\begin{proof}
	First Jensen's inequality, then \cref{eq:data-processing-vector}.
\end{proof}

\subsection{Proof of \cref{lem:dec-Delta_f-comparison}}\label{app:proof-lem-dec-Delta_f-comparison}
\begin{proof}[Proof of \cref{lem:dec-Delta_f-comparison}]
Note that
\begin{align*}
    \acdec_\eps(f) &= \min_{\mu \in \cP(\Pi)} \max_{\nu \in \cP{(\cH)}} \mu \Delta \nu \st \mu I_f \nu \leq \eps^2\\
    &= \min_{\mu \in \cP(\Pi)} \max_{\nu \in \cP{(\cH)}} \mu \Delta_f \nu + \sum_{g \in \cH} \sum_{\pi \in \Pi} \nu_g \mu_\pi (r_f(\pi) - r_g(\pi) ) \st \mu I_f \nu \leq \eps^2\\
    &\leq \min_{\mu \in \cP(\Pi)} \max_{\nu \in \cP{(\cH)}} \mu \Delta_f \nu + \sqrt{\mu I_f \nu} \st \mu I_f \nu \leq \eps^2\\
    &\leq \eps + \min_{\mu \in \cP(\Pi)} \max_{\nu \in \cP{(\cH)}} \mu \Delta_f \nu \st \mu I_f \nu \leq \eps^2\\
    &\leq \eps + \acfdec_\eps(f) \st \mu I_f \nu \leq \eps^2,
\end{align*}
where the first inequality is by \cref{lem:data-processing-with-envs}.
Also, by lower bounding the sum in the second inequality by $- \sqrt{\mu I_f \nu}$ we get left inequality.
\end{proof}

\subsection{Proof of \cref{lem:d*-to-dc}}\label{app:d*-to-dc-proof}
\begin{proof}[Proof of \cref{lem:d*-to-dc}]
	For the first inequality, using the definition of $\acdec_\eps(f, \Delta_f)$ and the AM-GM inequality:
	\begin{align}
		\acfdec_\eps(f)
		&= \min_{\lambda \ge 0} \max_{\nu \in \sP(\cH)} \min_{\mu \in \sP(\Pi)} \mu \Delta_f \nu - \lambda \mu I_f \nu + \lambda \epsilon^2 \nonumber \\
		&\le \min_{\lambda > 0} \max_{\nu \in \sP(\cH)} \min_{\mu \in \sP(\Pi)} \frac{(\mu \Delta_f \nu)^2}{4 \lambda \mu I_f \nu} + \lambda \epsilon^2 \\
		&= \min_{\lambda > 0} \frac{\Psi(f)}{4 \lambda} + \lambda \epsilon^2 = \eps \sqrt{\Psi(f)}\, .  \label{eq:info-ratio-inq}
	\end{align}
	Further, by \cref{eq:dc-inq} and \cref{asm:reward-data-processing} we have $\mu_\nu^\text{TS} \Delta_f \nu \le \sqrt{K \sum_{g, h \in \cH} \nu_g \nu_h e_{\pi^*_h} (r_g - r_f)^2} \le \sqrt{K \mu_f^\text{TS} I_f \nu}$, which gives $\Psi(f) \le K$. 
	Plugging this into \cref{eq:info-ratio-inq} gives the second inequality. 
\end{proof}

\subsection{Regret bound for \cref{alg:e2d} defined for $\Delta_f$ and $\acfdec_{\eps}$}
\begin{lemma}\label{lem:worst-case-deltaf}
	If \cref{asm:reward-data-processing} holds, then the regret of \AETD~(\cref{alg:e2d}) with $\Delta$ replaced with $\Delta_f$ is bounded as follows:
	\begin{align*}
		R_n &\leq  \max_{t \in [n], f \in \cH} \left\{\frac{ \acfdec_{\eps_t}(f)}{\eps_t^2}\right\} \left(\sum_{t=1}^n \eps_t^2 + \Est_n\right) + \sqrt{n \Est_n}
	\end{align*}
\end{lemma}

\begin{proof} The proof follows along the lines of the proof of \cref{thm:worst-case}. The main difference is that when introducing $\Delta_f$, we get a term that captures the reward estimation error:
%	\begin{align}
%		R_n &=  \EE[\sum_{t=1}^n \mu_t \Delta_{f^*} e_{f^*}]\\
%		&=  \sum_{t=1}^n \EE[\mu_t \Delta_{\hat f_t} e_{f^*} - \lambda_t (\mu_t I_{\hat f_t} e_{f^*} - \eps_t^2) + \lambda_t (\mu_t I_{\hat f_t} e_{f^*} - \eps_t^2) + \mu_t (r_{\hat f_t} - r_{f^*})]\\
%		&\leq \sum_{t=1}^n \EE[\max_{g \in \cH} \mu_t \Delta_{\hat f_t} e_{g} - \lambda_t (\mu_t I_{\hat f_t} e_{g} - \eps_t^2) + \lambda_t (\mu_t I_{\hat f_t} e_{f^*} - \eps_t^2) + \mu_t (r_{\hat f_t} - r_{f^*})]\\
%		&= \sum_{t=1}^n  \EE[\min_{\lambda \geq 0} \min_{\mu \in \sP(\Pi)} \max_{\nu \in \sP(\cH)}\mu \Delta_{\hat f_t} \nu  - \lambda (\mu I_{\hat f_t} \nu - \eps_t^2)  + \lambda_t (\mu_t I_{\hat f_t} e_{f^*} - \eps_t^2)]  \nonumber\\
%		&\qquad + \sum_{t=1}^n \EE[\mu_t (r_{\hat f_t} - r_{f^*})]
%	\end{align}
%	So far, we only introduced the saddle point problem by maximizing over $f^*$. The last equality is by our choice of $\lambda_t$ and $\mu_t$. Continuing, 
	\begin{align}
		R_n &\leq \sum_{t=1}^n  \EE[\acdec_{\eps_t}(\hat f_t) + \lambda_t (\mu_t I_{\hat f_t} e_{f^*} - \eps_t^2)] 
		+ \sum_{t=1}^n \EE[\mu_t (r_{\hat f_t} - r_{f^*})]\\
%		&\stackrel{(i)}{\leq} \sum_{t=1}^n \EE[\max \left(d_{\eps_t}^*(\hat f_t),  \frac{1}{\eps_t^2} d^*_{\eps_t}(\hat f_t) \mu_t I_{\hat f_t} e_{f^*} \right)] 
%		+   \sqrt{ n \sum_{t=1}^n \EE[ (\mu_t (r_{\hat f_t} - r_{f^*}))^2 ]}\\
%		&\stackrel{(ii)}{\leq} \sum_{t=1}^n \EE[\eps_t^{-2}d^*_{\eps_t}(\hat f_t)\big(\eps_t^2 +  \mu_t I_{\hat f_t} e_{f^*}\big)] 
%		+   \sqrt{n \sum_{t=1}^n \EE[\mu_t I_{\hat f_t} e_{f^*}]} \label{eq:worst-case-regret-proof-2}\\
		&\leq \max_{t \in [n]}  \max_{f \in \cH} \left\{\frac{1}{\eps_t^2} \acdec_{\eps_t}(f) \right\} \sum_{t=1}^n \big(\eps_t^2 +  \EE[\mu_t I_{\hat f_t} e_{f^*}]\big) + \sqrt{n \Est_n}
%		&= \max_{t \in [n]}  \max_{f \in \cH}\left\{\frac{1}{\eps_t^2} d^*_{\eps_t}(f) \right\} \left(\sum_{t=1}^n \eps_t^2  + \Est_n\right) + \sqrt{n \Est_n}
	\end{align}
	For the last inequality, we used Cauchy-Schwarz and \cref{asm:reward-data-processing} to bound the error term, 
	$$\sum_{t=1}^n \EE[\mu_t (r_{\hat f_t} - r_{f^*})] \leq \sqrt{n \sum_{t=1}^n  \EE[ (\mu_t (r_{\hat f_t} - r_{f^*}))^2 ]} \leq \sqrt{n \sum_{t=1}^n  \EE[\mu_t I_{\hat f_t} e_f]} = \sqrt{n \Est_n}$$
\end{proof}

\subsection{Generalized Information Ratio}\label{app:generalized-information-ratio}
The generalized information ratio \citep{lattimore2021mirror} for $ \mu \in \sP(\Pi)$, $\nu \in \sP(\cH)$, and $\alpha > 1$ is defined as 
\begin{align}
    \Psi_{\alpha, f} (\mu, \nu) = \frac{(\mu \Delta_f \nu)^\alpha}{\mu I_f \nu}
\end{align}
For $\alpha = 2$, we get the standard information ratio introduced by \citet{russo2014learning} with $\nu$ as a prior over the model class $\cM$. Define $\Psi_\alpha(f) = \max_{\nu \in \cH} \min_{\mu \in \sP(\Pi)}\Psi_{\alpha,f}(\mu, \nu)$. To upper bound $\acdec_\eps$, we have the following lemma.
\begin{lemma}
For the reference model $f$, the ac-dec can be upper bounded as
    \begin{align}
    \acdec_{\epsilon}(f) \leq \min_{\lambda > 0}  \big\{\lambda^{\frac{1}{1-\alpha}} \alpha^{\frac{\alpha}{1-\alpha}} (\alpha-1) \Psi_\alpha(f)^{\frac{1}{\alpha-1}} + \lambda\epsilon^2 \big\} \label{lem:d-psi}
\end{align}
for $\alpha > 1$.
\end{lemma}
\begin{proof}
    We start by noting that for $x_1,\dotsc,x_{\alpha} \ge 0$, from AM-GM we have that $\alpha (x_1 \cdot x_2 \dotsm x_\alpha)^{1/\alpha} \le x_1 + \cdots + x_\alpha$. Substituting $x_2 = x_3 = \cdots x_\alpha$, we get $\alpha \cdot x_1^{\frac{1}{\alpha}} \cdot x_2^{\frac{\alpha-1}{\alpha}} - x_1 \le (\alpha-1)x_2$. Writing $x_1 = \lambda \mu I_f \nu $ and $x_2 = \alpha^{\frac{\alpha}{1-\alpha}} \left( \frac{(\mu \Delta \nu)^\alpha}{\lambda \mu I_f \nu} \right)^{\frac{1}{\alpha-1}}$ and using the previous inequality with the $\acdec_\eps$ program gives the result. 
\end{proof}
The information ratio $ \Psi_{\alpha, f}(\mu, \nu)$ can be thought of as the Bayesian information ratio in \citep{russo2014learning} where the expectation is taken over the distribution $\nu$ of possible environments. However, for information gain $I_f$, \citep{russo2014learning} use entropy difference in the posterior distribution of $\pi^*_f$ before and after the observation is revealed. 

% If we can bound $\Psi_{\nu^*_\lambda, \alpha}(\mu)$ for all $\nu^*_{\lambda}$, for example using the information-directed sampling framework \citep{russo2014learning, kirschner2018information}, we get a bound on $\dec^a_\epsilon(f)$ by minimizing over $\lambda$.

\subsection{Proof of \cref{lem:linear-dec}}\label{app:proof-lem-linear-dec}

\begin{proof}[Proof of \cref{lem:linear-dec}]
	\begin{align}
		\acfdec_{\eps}(\hat f_t) &= \min_{\mu \in \sP(\Pi)} \min_{\lambda \geq 0} \max_{b \in \Pi} \max_{g \in \cH}  \ip{\phi_b, g} - \ip{\phi_\mu, \hat f_t} - \lambda \|g - \hat f_t \|_{V(\mu)}^2  + \lambda \eps^2 \nonumber \\
		&= \min_{\lambda \geq 0} \min_{\mu \in \sP(\Pi)} \max_{b \in \Pi} \ip{\phi_b - \phi_\mu, \hat f_t} + \frac{1}{4\lambda} \|\phi_b\|_{V(\mu)^{-1}}^2 + \lambda \eps^2 \label{eq:convex-dec} \,.
		%	&= \min_{\lambda \geq 0} \min_{\mu \in \sP(\Pi)} \max_{\omega \in \sP(\Pi)} \sum_{b \in \Pi} \omega_b\big(\ip{\phi_b - \phi_\mu, \hat f_t} + \frac{1}{4\lambda} \|\phi_b\|_{V(\mu)^{-1}}^2 \big) + \lambda \eps^2
	\end{align}
	The first equality is by definition, and the second equality follows from solving the quadratic maximization over $g \in \cM = \bR^d$. To show that the problem is convex in $\mu$, note that taking inverses of positive semi-definite matrices $X,Y$ is a convex function, i.e. $((1-\eta) X + \eta Y)^{-1} \preceq (1-\eta) X^{-1} + \eta Y^{-1}$. In particular, $V((1-\eta) \mu_1 + \eta \mu_2)^{-1} \preceq (1-\eta) V(\mu_1)^{-1} + \eta V(\mu_2)^{-1}$.  With this the claim follows.
\end{proof}

\section{Convex Program for Fixed $\lambda$} 
\label{app:lp-fixed-lambda}

% \paragraph{KKT for fixed $\lambda$}
Take \cref{eq:convex-dec} and fix $\lambda > 0$. 
Then we have the following saddle-point problem:
\begin{align*}
&\min_{\mu \in \sP(\Pi)} \max_{b \in \Pi} \ip{\phi_b - \phi_\mu, \hat f_t} + \frac{1}{4\lambda} \|\phi_b\|_{V(\mu)^{-1}}^2 + \lambda \eps^2 \\
&= \lambda \eps^2 + \min_{\mu \in \sP(\Pi)} \max_{b \in \Pi} \ip{\phi_b - \phi_\mu, \hat f_t} + \frac{1}{4\lambda} \|\phi_b\|_{V(\mu)^{-1}}^2
\end{align*}
Up to the constant additive term, this saddle point problem is equivalent to the following convex program
\begin{align}
	\min_{y \in \bR, \mu \in \bR^{\Pi}} y \st  y &\geq  \ip{\phi_b - \phi_\mu, \hat f_t} + \frac{1}{4\lambda} \|\phi_b\|_{V(\mu)^{-1}}^2 \quad \forall b \in \Pi \nonumber \\
	\eye \mu &= 1 \nonumber\\
	\mu_\pi &\geq 0\quad \forall \pi \nonumber
\end{align}

\section{Experiments}\label{app:experiments}
All experiments below were run on a semi-bandit problem with a "revealing action", as alluded to in the paragraph below \cref{lem:dec-bounds-linear}.
Specifically, we assume a semi-bandit model where $\cH = \bR^d$ and the features are $\phi_\pi \in \bR^d$.
For an instance $f^* \in \cM$, the reward function is $r_{f^*} = \langle \phi_\pi, f^* \rangle$ for all $\pi \in \Pi$.
There is one revealing (sub-optimal) action $\hat \pi \neq \pi^*_{f^*}$.
% The revealing action $\hat \pi$ is set such that it is a sub-optimal action (i.e. $\max_{\pi \in \Pi} \langle \phi_\pi, f^* \rangle > \langle \phi_{\hat \pi}, f^* \rangle$).
The observation for any action $\pi \neq \hat \pi$ is
\begin{align}
M_{f^*}(\pi) = \mathcal{N}(\langle \phi_\pi, f^* \rangle, 1)
\end{align}

Define $M_{\hat \pi} = [\phi_{\pi_1}, \dots, \phi_{\pi_{|\Pi|}}]^\top$. 
Then the observation for action $\hat \pi$ is
\begin{align}
M_{f^*}(\hat \pi) = \mathcal{N}(M_{\hat \pi} f^*, \eye_d)
\end{align}

Thus, the information for any action $\pi \neq \hat \pi$ is
\begin{align}
I_{f}(g, \pi) = \frac{\sigma^2}{2} \langle \phi_\pi, g - f \rangle^2
\end{align}

while the information for action $\hat \pi$ is 
\begin{align}
I_{f}(g, \hat \pi) = \frac{\sigma^2}{2} \|M_{\hat \pi} (g - f)\|^2 = \frac{\sigma^2}{2} \sum_\pi \langle \phi_\pi, g - f\rangle^2
\end{align}

For this setting $\Est_n \le \cO(d \log(n))$ (see \cref{app:ridge}).

\subsection{Experimental Setup}
Our main objective is to compare our algorithm \AETD~to the fixed-horizon \ETD~algorithm by \citet{foster2021statistical}. 
\AETD~and \ETD~were implemented by using the procedure described in \cref{sec:comp-aspects}.
Both \AETD~and \ETD~need to solve the inner convex problem in \cref{lem:linear-dec}.
To do so we use Frank-Wolfe \citep{frank1956algorithm,dunn1978conditional,jaggi2013revisiting} for 100 steps and warm-starting the optimization at the solution from the previous round, $\mu_{t-1}$. 
For \AETD~we further perform a grid search over $\lambda \in [0, \max_{g \in \cM} \eps^{-2} \acfdec (g)]$ (with a discretization of $50$ points) to optimize over lambda within each iteration of Frank-Wolfe. 
For both the \ETD~and \AETD~algorithm we used the version with the gaps $\Delta$ replaced with $\Delta_f$, since we noticed that both algorithms performed better with $\Delta_f$.
For \ETD, the scale hyperparameter $\lambda$ was set using $\lambda = \sqrt{\frac{n}{4 \log(n)}}$ as mentioned in \citet[Section 6.1.1]{foster2021statistical}.
While for \AETD~we set the hyper-parameter $\eps_t^2 = d / t$. 
Further, we compare to standard bandit algorithms: Upper Confidence Bound (\UCB)~and Thompson Sampling (\TSa)~\citep{lattimore2020bandit}.

\subsection{Experiment 1}
In this experiment, we aim to demonstrate the advantage of having an anytime algorithm.
Specifically, we tune $\lambda$ in the \ETD~algorithm for different horizons $n=200, 500, 1000, 2000$, but run it for a fixed horizon of $n=2000$.
As such, we expect our algorithm \AETD~to perform better than \ETD~when $\lambda$ was tuned for the incorrect horizons (i.e. $n=200, 500, 1000$). 
The feature dimension is $d=3$.
The number of decisions is $|\Pi| = 10$.
We generated the features $\phi_\pi$ for each $\pi \in \Pi$ and parameter $f^* \in \mathbb{R}^d$ randomly at the beginning and then kept them fixed throughout the experimentation.
%The parameter $f^* \in \mathbb{R}^d$ was selected randomly.
100 independent runs were performed for each algorithm.
%\todok{Mention linear bandits with side observations? Vlad: already mentioned above and everything is run for this setting...}

The results of the experiment can be seen as the left plot in \cref{fig:semi-bandit}. 
As expected, our algorithm \AETD~performs better than \ETD~(for $n=200, 500, 1000$). 
This indicates that the \ETD~algorithm is sensitive to different settings of $\lambda$, which is problematic when the horizon is not known beforehand.
Whereas our \AETD~algorithm performs well even when the horizon is not known.

\begin{figure}[t]
	\includegraphics{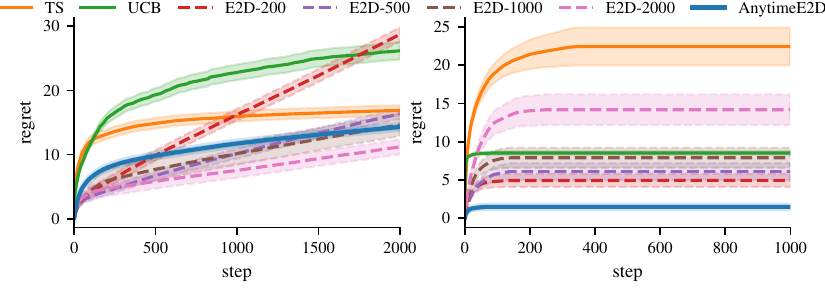}
	\caption{
    Running $\AETD$, TS, UCB, and \ETD~optimized for different horizons $n \in \{200, 500, 1000, 2000\}$.
    Left: The result for horizon $n=2000$, and the feature space dimension $d=3$.
    Right: The result for horizon $n=1000$, and the feature space dimension $d=30$.
    }
\label{fig:semi-bandit}
\end{figure}

\subsection{Experiment 2}
In this experiment, we investigate the case when $n < d^4$. 
As pointed out below \cref{lem:p-dec bound}, we expect improvement in this regime as the regret bound of our algorithm is $R_n \leq \min \{d \sqrt{n}, d^{1/3} n^{2/3}\}$, while the default, fixed-horizon \ETD~ algorithm cannot achieve these bounds simultaneously and one has to pick one of $d\sqrt{n}$ or $d^{1/3} n^{2/3}$ beforehand for setting the scale hyperparameter $\lambda$.
It is standard that the choice of $\lambda$ is made according to the $d \sqrt{n}$ regret bound for \ETD~\cite{foster2021statistical}(which is not optimal when $n <\!\!< d^4$), especially, if the horizon is not known beforehand.
Thus, we set the horizon to $n=1000$ and the dimension of the feature space to $d=30$, which gives us that $n=1000 <\!\!< 810000 = d^4$.
The rest of the setup and parameters are the same as in the previous experiment except for the features $\phi_\pi$ and $f^*$ which are again chosen randomly in the beginning and then kept fixed throughout the experiment.

The results of the experiment can be seen as the right plot in \cref{fig:semi-bandit}. 
As expected, our algorithm \AETD~performs better than \ETD, \UCB, and \TSa.
This indicates that indeed, \AETD~is likely setting $\lambda$ appropriately to achieve the preferred $d^{1/3}n^{2/3}$ regret rate for small horizons. 
The poor performance of the other algorithms can be justified, since \ETD~is optimized based on the worse $d \sqrt{n}$ regret rate (for small horizons), while the \UCB~and \TSa~algorithms are not known to get regret better than $d \sqrt{n}$.

\end{document}